%% file: main.tex
\documentclass{article}

\usepackage[main, final]{neurips_2026}

\usepackage[utf8]{inputenc}
\usepackage[T1]{fontenc}
\usepackage{hyperref}
\usepackage{url}
\usepackage{booktabs}
\usepackage{amsfonts}
\usepackage{amsmath}
\usepackage{amssymb}
\usepackage{amsthm}
\usepackage{nicefrac}
\usepackage{microtype}
\usepackage{xcolor}

\usepackage{graphicx}
\usepackage{algorithm}
\usepackage{algpseudocode}
\usepackage{xfrac}
\usepackage{paracol}
\usepackage{caption}
\usepackage{subcaption}
\usepackage{multirow}
\usepackage{mathtools}
\usepackage{tikz}
\usepackage{placeins}

\usepackage{arydshln}
\usepackage{soul}

\usepackage[capitalize]{cleveref}

\DeclareCaptionSubType{algorithm}

\usepackage{macros}
\usepackage{enumitem}

\newtheorem{lemma}{Lemma}

\newcommand\vsp{{\vspace*{-0.1cm}}}

\title{Beyond MMSE: Enhancing PnP Restoration\\ with ProxiMAP}

\author{
  Kenta Vert$^{1,}$\thanks{Equal contribution.} \quad
  Giacomo Meanti$^{1,2,3,*}$ \quad
  Scott Pesme$^{1}$ \quad
  Michael Arbel$^{1}$ \quad
  Julien Mairal$^{1}$ \\[4pt]
  \normalfont $^{1}$Univ. Grenoble Alpes, Inria, CNRS, Grenoble INP, LJK \\
  \normalfont $^{2}$MaLGa Centre, DIBRIS, Universit\`a di Genova \\ $^{3}$MMS,
  \normalfont Italian Institute of Technology, Genova, Italy \\[2pt]
  \texttt{\{firstname.lastname\}@inria.fr}
}

\begin{document}

\maketitle

\begin{abstract}
Plug-and-Play (PnP) methods for imaging inverse problems replace the proximal operator of the prior, a maximum a posteriori (MAP) denoiser, with a pretrained minimum mean squared error (MMSE) denoiser. MMSE estimates are oversmoothed, motivating recent work to target the MAP with diffusion-model scores. We show that MAP-targeting iterations converge to cartoon-like images with learned scores, but to realistic ones with the exact scores of a Gaussian mixture. The failure thus seems to come from score errors, and it only appears near convergence.
We therefore propose \trueprox/, an iterative MAP approximation that queries the denoiser at the residual noise level of the current iterate. This keeps the denoiser in-distribution and slows the iterations down before they reach the cartoon regime.
On deblurring, super-resolution, inpainting, HDR and phase retrieval, replacing the MMSE denoiser with \trueprox/ in DPIR, DiffPIR and DAPS improves perceptual quality in most experiments, sometimes at the expense of PSNR. To make \trueprox/ cheaper, we also propose a one-step distillation and a hybrid schedule that applies it only at the last PnP iteration. In most cases, one or the other retains much of the improvement at a fraction of the cost.
\end{abstract}

\section{Introduction}\label{sec:intro}
\input{sec/intro.tex}

\section{Related Work}\label{sec:related}
\input{sec/related.tex}

\section{The \trueprox/ Algorithm}\label{sec:algorithm}
\input{sec/algorithm.tex}

\section{Experiments}\label{sec:experiments}
\input{sec/experiments.tex}

\section{Conclusion}
PnP methods often rely on an MMSE denoiser, producing the well-known oversmoothing of posterior-mean estimators. We showed that simply targeting the MAP with diffusion-model scores does not resolve this: with learned scores, MAP-targeting iterations drift toward unrealistic regions of the learned distribution, and a controlled comparison with exact GMM scores confirms that this is a property of the learned score, not of MAP estimation itself. Our \trueprox/ algorithm addresses this issue: used as a replacement for MMSE denoisers, it consistently sharpens reconstructions across multiple inverse problems and datasets. We manage to keep improvements brought by \trueprox/ at a reasonable cost using a distilled model or a hybrid variant in PnP.

\paragraph{Limitations and future work.}
\trueprox/ iterations are driven by the score implied by the chosen denoiser through \eqref{eq:tweedie}, so its output is bounded by the accuracy of that score, while an MMSE estimate is an average and more forgiving of an imprecise distribution model. \trueprox/ extracts details the denoiser already contains rather than adding any, so its performance follows the quality of this network. Our experiments use a different pretrained network for each dataset, together with flow-matching models on CIFAR-10 (\cref{app:memorizing}) and MNIST (\cref{sec:diagnosis}), but we do not treat the denoiser as an experimental variable in its own right; where we have probed it, the dependence is noticeable, with a smaller non-diffusion denoiser of the \texttt{DRUNet} \citep{zhang2021plug} family degrading results for \diffpir/. A systematic study of how performance tracks prior quality is left to future work.
A second limitation is that \trueprox/ does not converge to a well-defined fixed point (\cref{app:discussion}): its behaviour follows from the schedule rather than from any limit of the iteration, and we do not have a criterion for when an iterate carrying estimation error has left the denoiser's training distribution. Finally, a theoretical account of the phenomenon we exploit would be valuable but appears difficult: the object one would want to characterize is the MAP of the data distribution, which is inaccessible, while what is computable is the MAP of the learned one, and the two differ in ways that depend on the implicit biases of diffusion training---biases which, to our knowledge, have not been characterized.

\paragraph{Broader impact.}
\trueprox/ improves the perceptual quality of reconstructions produced by Plug-and-Play methods across a range of inverse problems. Because the underlying diffusion models are pretrained on natural-image datasets, reconstructions inherit the biases and content distribution of those datasets; this is a property of the diffusion priors used rather than of \trueprox/ itself. More fundamentally, our analysis (\cref{sec:diagnosis}) shows that diffusion-based MAP estimators recover what is most likely under the \emph{learned} distribution rather than under the data distribution. This caveat is particularly relevant for safety-critical applications such as medical or scientific imaging, where downstream decisions should not rely on visual plausibility alone.

\begin{ack}
This work was supported by the ERC grant number 101087696 (APHELAIA project) and by the ANR project BONSAI (grant ANR-23-CE23-0012-01). This work was granted access to the HPC resources of IDRIS under the allocation 2025-AD011016894  made by GENCI.
\end{ack}

\newpage

\bibliographystyle{plainnat}
\bibliography{main}

\appendix
\input{sec/appendix.tex}

\end{document}

%% file: sec/intro.tex
Inverse problems are ubiquitous in imaging, sharing the mathematical structure $y = \cA(x) + \varepsilon$, where $x$ is the unknown clean image, $\cA$ a known degradation operator (\emph{e.g.}, blur, subsampling), and $\varepsilon$ measurement noise. This common structure has motivated general-purpose frameworks such as Plug-and-Play (PnP)~\citep{plugandplay,chan2016plug} and Regularization by Denoising (RED)~\citep{romano17red}, which combine optimization with powerful image priors.

Adopting a probabilistic perspective, recovery of $x$ from $y$ is formulated as posterior maximization:
\begin{equation}\label{eq:map}
x^* = \argmax_x p(x\mid y)
= \argmax_x \log p(y\mid x) + \log p(x),
\end{equation}
or, equivalently, as $x^* = \argmin_x f(x, y) + \lambda g(x)$, with $f$ a data-fidelity term, $g$ a regularizer reflecting prior knowledge, and $\lambda>0$ balancing the two. 
Such composite problems, when $f$ is strongly convex but $g$ is not, can be solved with proximal splitting methods such as ADMM and HQS~\citep{combettes11,parikh14proximal} which rely on the proximal operator
\begin{equation}\label{eq:proxg}
\prox_{\tau g}(z) := \argmin_x \frac{1}{2}\|x - z\|^2 + \tau g(x).
\end{equation}
The key insight underlying PnP is that when $g = -\log p$, the proximal operator~\eqref{eq:proxg} is the maximum a posteriori (MAP) denoiser of a Gaussian denoising problem. 
This motivates replacing $g$ with a denoiser $D_\theta$ which can be a deep neural network trained on large natural image datasets~\citep{zhang2017learning} in order to capture an expressive prior. 
In practice, however, $D_\theta$ is trained with a squared-error loss $\mathbb{E}_{x,\varepsilon} \| D_\theta(x + \sigma \varepsilon) - x \|^2$, and the trained model will approximate the MMSE estimator $\mathbb{E}[X \mid Y = x + \sigma \varepsilon]$ rather than the MAP that~\eqref{eq:map} actually requires.  The practical consequence is the well-known \textit{regression-to-the-mean} oversmoothing of MMSE-based reconstructions.

\vspace{-.8em}
\paragraph{Why MAP-targeting fails, and what to do instead.}
A natural way to close this gap is to estimate the MAP denoiser directly. 
\citet{laumont2023maximum} and \citet{pesme2025map} have shown convergence properties of a smoothed version of gradient descent on the proximal objective, using scores $\nabla_x \log p_\sigma(x)$ and annealing $\sigma\to 0$. %
With \emph{exact} scores and log-concave $p$, this procedure recovers the true MAP. In Section~\ref{sec:diagnosis}, we use a Gaussian mixture model, where the score is available in closed form, to verify that the recovery is clean. 
However, when the scores are learned from data, the procedure converges to a different object: the MAP of the learned distribution, whose high-density regions, due to the implicit biases of diffusion training, do not coincide with those of the true data distribution. The resulting images are cartoon-like (also reported by~\citet{karczewski25diffusion}), and better reconstructions are obtained by stopping short of convergence.
We turn this observation into a design principle for MAP approximation under inexact scores: at every step, the iterate's residual noise level should match the noise level at which the denoiser is queried. 
A learned denoiser is reliable only on inputs whose noise statistics match its training distribution; outside that regime its score estimate degrades, which leads to the failure mode above. 
We introduce \trueprox/, an iterative MAP approximation that follows this principle. 
The matching constraint determines the noise schedule, keeps the denoiser in-distribution throughout, and yields implicit early stopping preventing drift into the failure regime. The result is sharper than the MMSE estimator, as illustrated in \cref{fig:map-v-mmse}.

\begin{figure}[t]
    \centering
    \includegraphics[width=\textwidth]{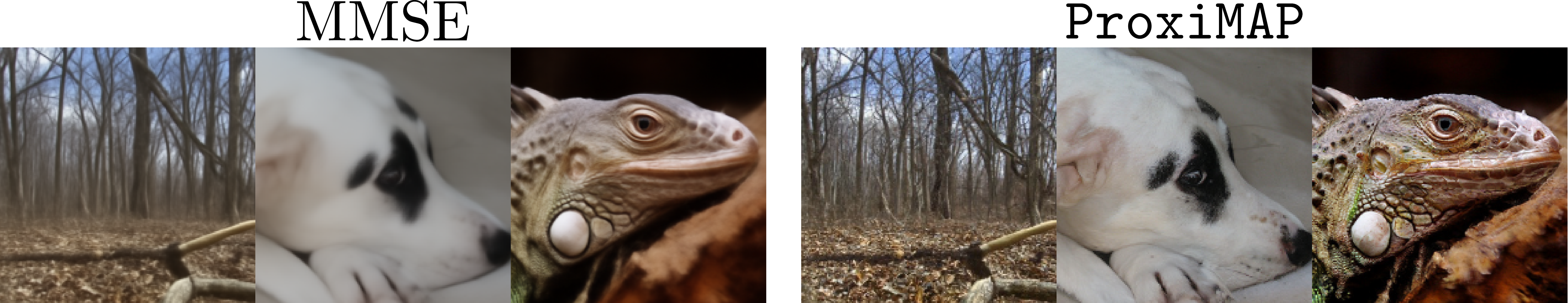}
    \caption{Conditional MMSE estimates (left) versus \trueprox/ (right) on three crops at noise level $\sigma_y = 0.2$. \trueprox/ recovers fine textures and sharp edges that the MMSE estimator smooths away. We propose to leverage this improvement for PnP image reconstruction.}
    \label{fig:map-v-mmse}
\end{figure}

\trueprox/ is a modular drop-in replacement for MMSE denoisers in standard PnP algorithms. We integrate it into three state-of-the-art PnP algorithms (DPIR, DiffPIR, DAPS) and evaluate it on deblurring, inpainting, super-resolution, phase retrieval and HDR reconstruction, on both FFHQ and ImageNet: it improves LPIPS in 23 of the 28 (algorithm, task, dataset) settings and is tied in 4, while its effect on PSNR depends on the host algorithm. We also explore two ways to reduce the computational cost while keeping much of the improvement in most cases, namely a one-step distillation of \trueprox/ and a hybrid variant that uses the MMSE denoiser in the early PnP iterations and \trueprox/ only at the end.

In summary, our main contributions are:
\begin{itemize}[leftmargin=*, itemsep=0pt, topsep=0pt, parsep=0pt]
    \item \textbf{A diagnosis of MAP estimation under learned priors.} Via a controlled comparison between a diffusion model and a Gaussian mixture model where exact scores are available, we show that the cartoon-like artifacts produced by naive MAP iteration stem from score approximation error rather than from MAP estimation itself. As a consequence, better reconstructions are obtained by stopping short of convergence.
    \item \textbf{An effective algorithm derived from a single principle.} \trueprox/ is an iterative MAP approximation whose noise and step-size schedules follow from matching the iterate's residual noise to the denoiser's training noise at every step. This keeps the denoiser in-distribution and yields implicit early stopping, providing better approximation under inexact scores.
    \item \textbf{Systematic improvements over PnP baselines} across multiple algorithms, tasks and datasets, established via Bayesian hyperparameter optimization to ensure fair comparison.
    \item {\textbf{Methods to substantially reduce computing costs} in PnP and denoising setting: a single-step distillation model that performs at a fraction of the cost of \trueprox/ and a hybrid variant that improves \trueprox/'s output in some PnP settings.}
\end{itemize}

%% file: sec/related.tex
\vsp
\noindent
\textbf{Plug-and-play (PnP) methods.} PnP methods~\citep{plugandplay} replace the proximal operator in iterative optimisation algorithms with a generic denoiser~$D_\theta(y)$, typically parameterised by the noise level $\sigma$. This decoupling of the data-fidelity term from the prior allows to leverage off-the-shelf denoisers as implicit image priors. A wide variety of denoisers have been used within this framework, including classical approaches such as BM3D~\citep{bm3d}, convolutional neural networks~\citep{meinhardt17,zhang2021plug,zhang2017learning}, and, more recently, diffusion models~\citep{graikos2022diffusion,diffpir} as well as flow matching models~\citep{pnpflow}. Different proximal optimisation schemes can be used, such as PGD~\citep{terris20}, ADMM~\citep{romano17red}, and HQS~\citep{zhang2017learning}. As the PnP literature is vast, we refer the reader to~\citet{hurault2023convergent,kamilov23pnp} for a comprehensive overview. Convergence of PnP algorithms has been established under various assumptions on the denoiser~\citep{sreehari16,gavaskar20,nair2021fixed,xu2020provable,hurault2022proximal}. However, standard deep denoisers approximate the MMSE estimator rather than the proximal operator of the true prior, so the recovered solution cannot in general be interpreted as a MAP estimate. 

\noindent
\textbf{MAP estimation with learned priors.} Several recent works have recognized the limitations of MMSE-based denoisers and proposed to estimate the proximal operator $\prox_{-\tau \ln p}$ or, equivalently, the MAP denoiser, directly. \citet{fang2023s} replace the standard MSE objective with a regularized proximal matching loss whose minimizer coincides, as the regularization parameter tends to zero, with the proximal operator. \citet{laumont2023maximum} introduce PnP-SGD, which performs stochastic gradient descent on a smoothed version of the proximal objective; by keeping the smoothing parameter fixed, their method approximates the proximal operator of the smoothed density rather than the true one. To overcome this limitation, \citet{pesme2025map} leverage diffusion-model score estimates $\nabla \ln p_{\sigma_k}$ across multiple noise levels and analyze a smoothed gradient descent algorithm with annealed $\sigma_k$, proving convergence to the true proximal operator under a log-concavity assumption.

A parallel line of work formulates inverse problems as MAP estimation directly within the diffusion-sampling framework. \citet{gutha2024map} cast the reverse conditional generation as optimization of a tractable MAP objective, and develop algorithms for inpainting; \citet{gutha2025vml} build on this with a variational mode-seeking loss whose per-step minimization guides samples toward the MAP estimate. \citet{lmaps} propose Local MAP Sampling, a covariance approximation that unifies Tweedie-moment-projected diffusion~\citep{boys2024} with optimization-based methods. These works produce end-to-end inverse problem solvers and do not address the problem we consider: enhancing existing PnP frameworks via a modular replacement of their inner denoising step.

\noindent
\textbf{Sampling-based methods.} An alternative to optimization-based PnP methods is to sample from the posterior $p(x \mid y)$ rather than seek its maximum. A prominent class are \emph{conditional diffusion models}, which modify the unconditional score $\nabla \ln p_\sigma(x)$ into the posterior score $\nabla \ln p_\sigma(x \mid y)$, then sample via the reverse diffusion SDE. Various approximations make this tractable: \citet{dhariwal2021diffusion} use a pretrained classifier to estimate $\nabla \ln p(y \mid x)$; \citet{jalal21} approximate $p_\sigma(y \mid x) \approx p(y \mid x)$ via a Gaussian likelihood; DPS~\citep{dps} approximates the posterior mean using Tweedie's formula; and \citet{boys2024} further approximate the posterior standard deviation. Similarly, \citet{kadkhodaie2020solving} draw high-density samples by annealing the noise level and taking gradient steps. All these methods target the posterior distribution rather than its mode. Relatedly, the InDI algorithm of \citet{delbracio23inversion} is an iterative restoration method linked to the probability-flow ODE via Tweedie's formula.

%% file: sec/algorithm.tex
We begin by showing that with learned diffusion models the MAP solution for the denoising problem drifts towards unrealistic regions in image-space. In \cref{sec:proximap-design} we derive \trueprox/ as a way of avoiding such drift %
and in \cref{sec:proximap-pnp}, we discuss \trueprox/ in the context of PnP frameworks.

\subsection{MAP-targeting fails with learned scores}\label{sec:diagnosis}
Once again, we wish to compute the proximal operator of a noisy image $y = x + \sigma_y\varepsilon$, $\varepsilon\sim\cN(0, I)$
\begin{equation}\label{eq:prox2}
\prox_{-\tau\log p}(y) := \argmin_x \frac{1}{2}\|x - y\|^2 - \tau \log p(x),
\end{equation}
where $p(x)$ is the unknown density of natural images and $\tau$ controls regularization strength. With access to the score $\nabla\log p(x)$ one could solve~\eqref{eq:prox2} by gradient descent. In practice we only have approximate scores $\nabla\log p_\sigma(x)$ of the smoothed densities $p_\sigma = p \ast \cN(0, \sigma^2 I)$, obtained from a pretrained diffusion model. 
As $\sigma \to 0$, the smoothed scores approach the true $\nabla\log p$, but the objective becomes increasingly non-smooth and gradient descent at fixed small $\sigma$ is unstable. \citet{pesme2025map} proposed to solve \cref{eq:prox2} using smoothed gradient descent with decreasing $\sigma_k$:
\begin{align}
    x_{k+1} &= x_k - \gamma_k \nabla_x \bigg( \frac{1}{2}\norm{x_k - y}^2 - \tau \log p_{\sigma_k}(x_k) \bigg),  \label{eq:gd-general2}
\end{align}
{The gradient in \eqref{eq:gd-general2} can be expressed through the denoiser by Tweedie's formula \eqref{eq:tweedie}, which relates the score of a Gaussian-smoothed density to the corresponding posterior mean:
\begin{equation}
  \nabla \log p_\sigma(x) \;=\; \frac{\mathbb{E}[X \mid X_\sigma = x] - x}{\sigma^2}.
  \label{eq:tweedie}
\end{equation}
Substituting \eqref{eq:tweedie} into \eqref{eq:gd-general2} eliminates the score in favour of the MMSE denoiser. \citet{pesme2025map} prove convergence to the MAP under log-concavity of $p$ and with the following schedule:}
\begin{equation}\label{eq:original-algo}
    \sigma_k^2 = \frac{\tau}{k+1},~\gamma_k = \frac{1}{k+2} \implies x_{k+1} = \frac{y}{k+2} + \frac{k+1}{k+2}\bE[X\mid X_{\sigma_k} = x_k].
\end{equation}
Although the assumption does not hold for natural images, \cref{eq:original-algo} is a direct and theoretically grounded way to target the MAP. 
We show below that replacing the exact MMSE (not available in practice) by a score-based approximation yields a failure mode at convergence.

\textbf{The cartoon effect.~~}
We instantiate~\eqref{eq:original-algo} with three pretrained denoisers: a flow-matching model on MNIST and unconditional diffusion models on FFHQ and ImageNet. Across all three (\cref{fig:cartoons}), the resulting iterates do not resemble typical image samples. On MNIST the MAP estimates progressively thicken into blob-like shapes as $\sigma_y$ increases, in contrast to the overly smooth MMSE estimates. On FFHQ and ImageNet, iterates exhibit the \emph{cartoon} appearance recently reported by~\citet{karczewski25diffusion} for pretrained diffusion models: flat colors, exaggerated edges, and unrealistic textures.

\begin{figure}
    \centering
    \includegraphics[width=0.45\linewidth]{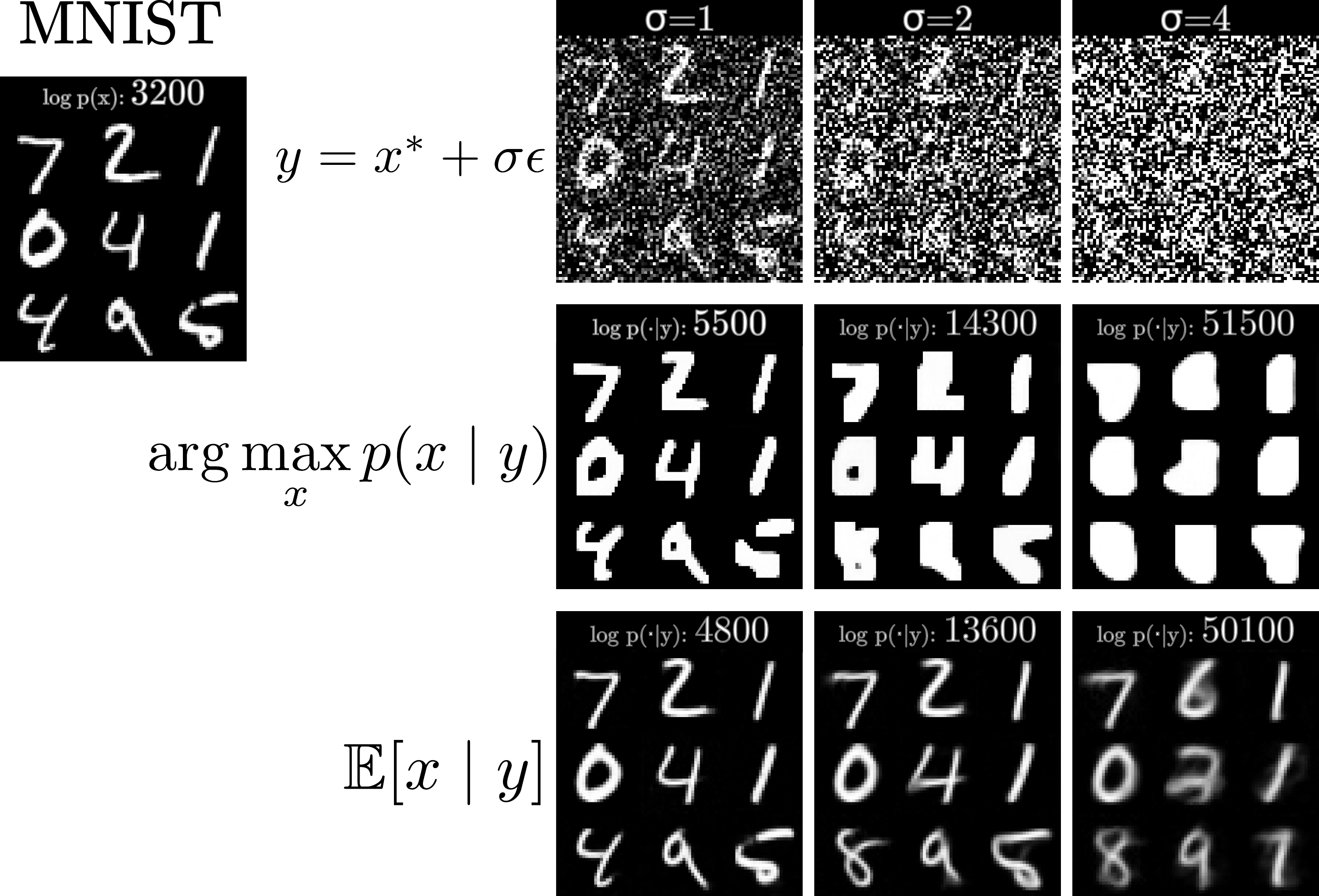}
    \includegraphics[width=0.49\linewidth]{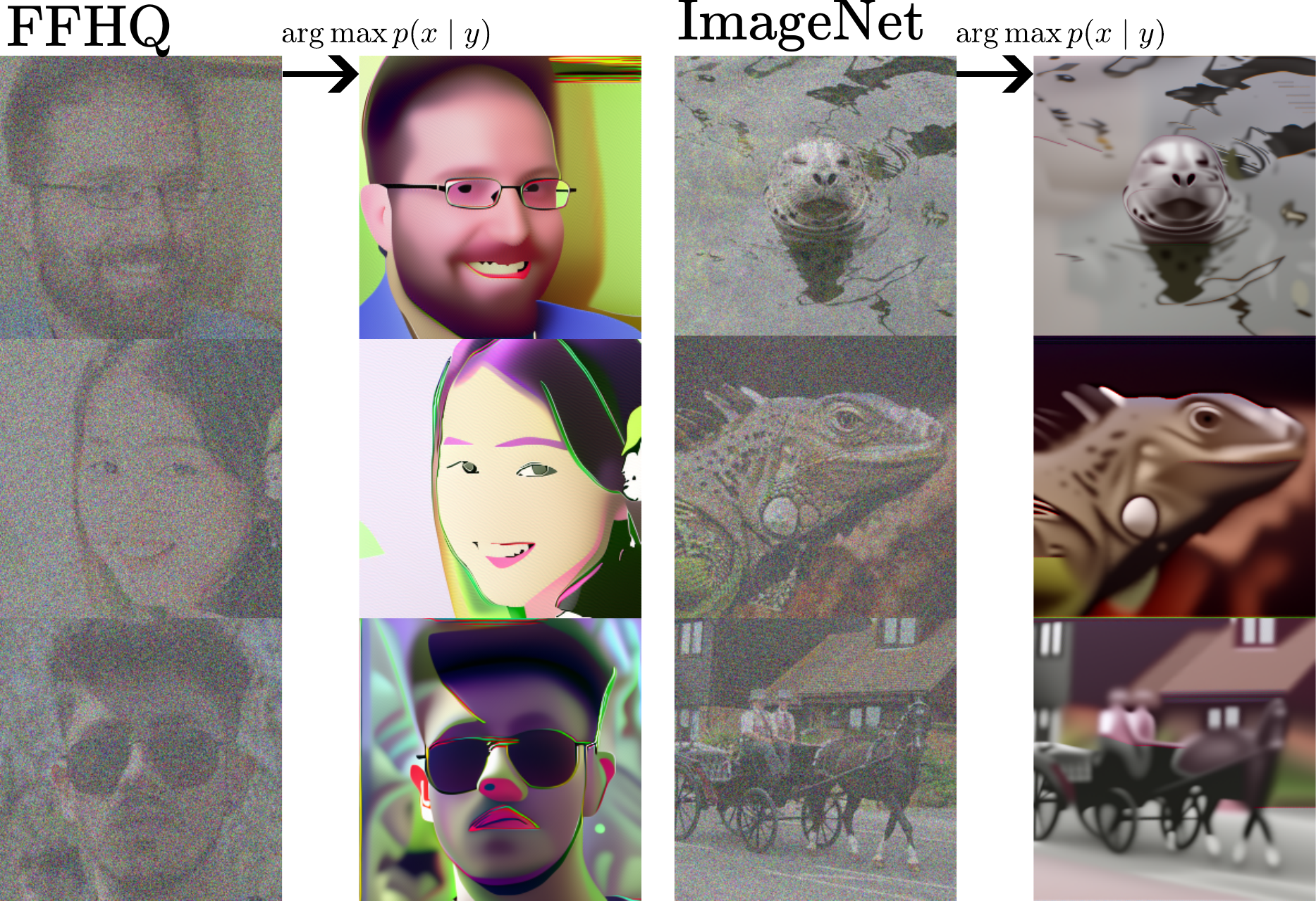}
    \caption{The MAP approximation procedure of \cref{eq:original-algo} with \emph{learned scores} yields cartoon-like outputs on natural-image datasets (right) and blob-like digits on MNIST (left). For MNIST we show the noisy data (top), along with MAP estimates (middle) and MMSE estimates (bottom). Conditional log-probability of the MAP estimates is always higher than that of the MMSE.}
    \label{fig:cartoons}
\end{figure}

Yet, the iterations do increase the learned posterior.
 To verify this, we computed the conditional log-likelihood of each image under the learned MNIST model using the augmented Probability-Flow ODE~\citep{song21scorebased,chen18neural}. The values, superimposed on \cref{fig:cartoons} (left), confirm that iterates of \cref{eq:original-algo} have higher probability than the MMSE estimates, despite looking unrealistic.
The cartoon images are what the model considers most likely given the noisy input. Now, either (a) high-density regions of the learned distribution genuinely contain non-typical images because of the implicit biases of diffusion model training~\citep{kambganguli}, or (b) gradient descent on the non-convex landscape gets stuck in poor local maxima. 
To distinguish these, we replace the trained model with a Gaussian mixture model (GMM) centered on the training set with small isotropic variance and for which $\nabla\log p_{\sigma_k}$ is available in closed form. 
The GMM has a non-convex probability distribution so any algorithmic pathology of GD due to non-convexity should manifest here as well. 
Running~\eqref{eq:original-algo} with the exact GMM score, however, yields realistic images (the modes of the GMM) independently of initialization, while the same algorithm with the diffusion-model score produces cartoon outputs (\cref{fig:gmm-comparison}). 
Non-convexity is therefore not the issue. The cartoon effect is a property of the \emph{learned} score, whose zeros are displaced from those of the true distribution by the implicit biases of diffusion model training. As a further check, in \cref{app:memorizing} we train a diffusion model in the \emph{memorizing regime}~\citep{biroli2024dynamical}---where the model has effectively recovered the empirical score---and find that MAP estimates from~\eqref{eq:original-algo} are then artifact-free. \cref{app:discussion} discusses broader implications for MAP estimation under learned priors.

\begin{figure}[t]
    \centering
    \includegraphics[width=0.8\linewidth]{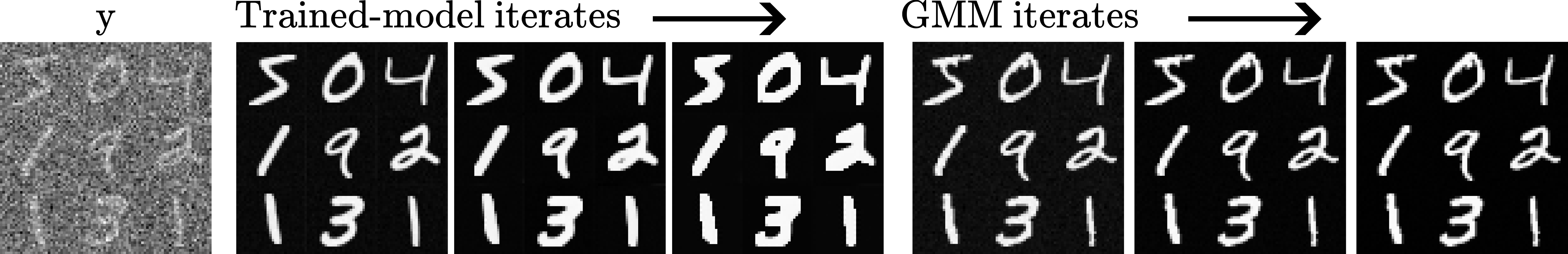}
    \caption{Iterates of \cref{eq:original-algo} with the score from a diffusion model (left) or from a GMM (right). With exact scores the iteration recovers clean modes; with learned scores it drifts to cartoon outputs.}\vspace*{-0.3cm}
    \label{fig:gmm-comparison}
\end{figure}

\textbf{Implication for algorithm design.~~}
The diagnosis suggests two things. First, the failure is not intrinsic to MAP estimation: with exact scores the procedure works. Second, the failure occurs as the iteration converges: early iterates of~\eqref{eq:original-algo} typically resemble plausible images.
In the next section we address how to stop short of convergence in a principled way.

\subsection{The \trueprox/ algorithm}\label{sec:proximap-design}

We wish to develop an algorithm which targets the MAP objective while keeping iterates in a regime where the residual noise in $x_k$ match the noise level $\sigma_k$ at which the denoiser is queried.
\trueprox/ achieves this by exploiting a structural property of the smoothed GD recursion~\eqref{eq:gd-general2}.

\paragraph{Tracking residual noise.} The noise in $y$ is i.i.d.\ Gaussian with standard deviation $\sigma_y$, which is the noise the denoiser is trained to remove. We track this Gaussian component along the iterations and treat each denoiser output as free of it, so that the Gaussian noise in $x_k$ remains a scaled copy of the noise in $y$. The output of a trained denoiser also carries an estimation error, which is structured and correlated with the signal, and which the recursion below does not model. The level $\sigma_k$ therefore describes only the Gaussian part of the residual, which is what the denoiser is conditioned on. Renoising-based solvers such as DiffPIR~\citep{diffpir}, DAPS~\citep{daps} and SNORE~\citep{renaud2024plug} rely on the same approximation, and \citet{kadkhodaie2020solving} track the effective noise level with a similar recursion. \citet{mohan2020robust} also show that trained denoisers remove most of the input noise by acting locally as projections onto a low-dimensional subspace. Empirically, \citet[App.~G]{renaud2024plug} measure the noise actually present in PnP iterates and find that the gap with the conditioning level decreases along the iterations when the schedule is matched to the iterates, while it stays large otherwise. Concretely, 
 {combining \eqref{eq:gd-general2} with \eqref{eq:tweedie} and analyzing the recursion term by term, we get:}
\begin{equation*}
    \underbrace{x_{k+1}}_{\sigma_{k+1}} = \bigg(1 - \gamma_k \Big( 1 + \frac{\tau}{\sigma_k^2} \Big) \bigg) \underbrace{x_k}_{\sigma_k} + \gamma_k \underbrace{y}_{\sigma_y} +  \frac{\tau \gamma_k}{\sigma_k^2} \underbrace{\bE[X\mid X_{\sigma_k} = x_k]}_{0}.
\end{equation*}
Since $x_0 = y$ and successive denoiser outputs carry no noise, the noisy component of every $x_k$ remains a multiple of $y$. Standard deviations therefore combine linearly, giving
\begin{equation}\label{eq:sigma_sched2}
    \sigma_{k+1} = \left( 1 - \gamma_k \Big( 1 + \frac{\tau}{\sigma_k^2} \Big) \right) \sigma_k + \gamma_k \sigma_y.
\end{equation}
This expression gives the residual noise level of $x_k$ as a function of the schedule.

\textbf{Choosing step-sizes and $\tau$.~~}
The smoothed objective $x \mapsto \frac{1}{2}\norm{x - y}^2 -\tau \log p_{\sigma_k}(x)$ at step $k$ is $L_{\sigma_k}$-smooth with $L_{\sigma_k} = 1 + \tau/\sigma_k^2$ \citep{pesme2025map}. The standard step-size for $L$-smooth objectives is $\gamma_k = \beta / L_{\sigma_k}$ with $\beta \in (0,1)$ a hyperparameter, which yields %
\begin{equation}\label{eq:tp-schedule}
    \sigma_{k+1} = ( 1 - \beta) \sigma_k + \frac{\beta}{1 + \tau/\sigma_k^2} \sigma_y,\qquad \gamma_k = \frac{\beta}{L_{\sigma_k}}.
\end{equation}
We then set $\tau = \sigma_y^2$. With this choice, \cref{eq:prox2} is the MAP estimator of $x$ given $y = x + \sigma_y \varepsilon$, which is also the operator that splitting schemes such as HQS require when they call a denoiser at noise level~$\sigma_y$.
This also ensures that $\sigma_k \to 0$, see \cref{app:lemma} in \cref{app:sec:proof-lemma}.

\textbf{Pseudocode and hyperparameters.~~}
The pseudocode for \trueprox/ is given in \cref{algo:fullprox}, with main hyperparameters $\beta$, and $K$. 
{Although \cref{fig:denoising} shows that LPIPS saturates beyond $K=8$, we decided to use $K=4$ in all PnP experiments to limit the cost of the nested loop.} We choose $\beta$ so that the final variance $\sigma_K^2$ attains a prescribed value given $\sigma_y$ and $K$. The final variance is therefore the only hyperparameter actively tuned in practice; a typical tuning range is $\sigma_{end} \in [0.0001, 0.1]$ {with a log-scale}. Finally, we return the last MMSE estimate $\hat{x}^{\mathrm{MMSE}}_{K} = D_\theta(x_{K-1}, \sigma_{K-1})$ since it provides a final noise-free output, whereas $x_K$ still contains a residual scaled version of $y$. {If $\sigma_K < \sigma_{end}$ we return the MMSE estimate at $\sigma_K$ and stop the algorithm early.}
\begin{algorithm}[h]
\caption{The \trueprox/ algorithm.
Input: noisy image $y$ (noise variance $\sigma_y^2$) and a trained denoiser $D_\theta$.
Output: approximate MAP denoising solution.}
\label{algo:fullprox}
\begin{algorithmic}[1]
\Require $\sigma_y, y, \beta, K$, denoiser $D_\theta$
\State $(x_0, \sigma_0) \gets (y, \sigma_y) $
\State $\tau \gets \sigma_y^2$ and $\gamma_0 \gets \frac{\beta}{1 + \sfrac{\tau}{\sigma_{0}^{2}}}$
\For{$k = 0, \dots, K - 1$}
    \State $\hat{x}^{\mathrm{MMSE}}_{k+1} \gets D_\theta(x_k, \sigma_k)$
    \State $x_{k+1} \gets (1 - \gamma_k ( 1 + \sfrac{\tau}{\sigma_k^2}))x_k + \gamma_k x_0 + \frac{\tau \gamma_k}{\sigma_k^2} \hat{x}^{\mathrm{MMSE}}_{k+1}$
    \State $\sigma_{k+1} \gets (1 - \gamma_k ( 1 +  \sfrac{\tau}{\sigma_k^2})) \sigma_k +\gamma_k \sigma_y$
    \State $\gamma_{k+1} \gets \frac{\beta}{1 + \sfrac{\tau}{\sigma_{k+1}^{2}}}$
\EndFor
\State \Return $\hat{x}^{\mathrm{MMSE}}_{K}$ \Comment{final denoiser call applied to $x_{K-1}$ at noise level $\sigma_{K-1}$}
\end{algorithmic}
\end{algorithm}

\textbf{\trueprox/ as implicit early stopping.~~}
Let $\hat{\sigma}(x_k)$ denote the residual noise level of the iterate $x_k$, tracked as above. The schedule of \cref{eq:original-algo}, chosen so that the iteration provably converges to the MAP for log-concave priors, queries the denoiser above this level. Indeed, $\hat{\sigma}(x_k) = \sigma_y/(k+1)$, whereas $\sigma_k = \sigma_y/\sqrt{k+1}$ when $\tau = \sigma_y^2$, so the two only agree at $k = 0$. The schedule of \cref{eq:tp-schedule} instead enforces $\sigma_k = \hat{\sigma}(x_k)$ for all $k$, which means the denoiser is always queried at the residual noise level of its input. In addition, the step size $\gamma_k$ shrinks with $\sigma_k$, so the iterations slow down as they approach the data manifold, and the output changes little after a few steps (\cref{fig:denoising}). \trueprox/ is thus designed to stop before reaching the learned MAP rather than to approximate it more accurately.

\subsection{\trueprox/ in PnP algorithms}\label{sec:proximap-pnp}

Having established that \trueprox/ performs effectively as a MAP denoiser, we next integrate it into PnP inverse problem solvers by replacing the standard MMSE estimator in their inner loop. {This results in an algorithm with two nested loops, the outer loop being the PnP algorithm; the inner loop is \trueprox/. Whenever the outer loop would call a denoiser, we call \trueprox/ to return our MAP approximation instead of the MMSE. This outer loop carries its own noise schedule, which supplies the noise level to initialize the \trueprox/ loop. The inner loop then updates its parameters as shown in \ref{algo:fullprox} and returns the cleaned image to be used by the outer loop.}
Among many existing PnP algorithms~\citep{kamilov23pnp,li25survey}, we select three representatives:
\begin{itemize}[leftmargin=*, itemsep=0pt, topsep=0pt, parsep=0pt]
    \item \dpir/~\citep{zhang2021plug} is based on half-quadratic splitting with a schedule that yields good results in few iterations. 
    \item \diffpir/~\citep{diffpir} improves the performance of \dpir/ on generative tasks such as block inpainting by introducing a \emph{noise-addition} step on top of existing data-fidelity and denoising steps. The added noise goes to 0 during optimization.
    \item \daps/~\citep{daps} replaces the denoising step with an iterative procedure which samples from $p(x\mid x_{\sigma_t})$. 
    Since the inner loop of \daps/ already requires multiple evaluations of the model, \trueprox/ does not introduce any computational overhead.
\end{itemize}
For the first two, we keep the same iteration and scheduling and replace calls to the denoiser $D_\theta$ with calls to \trueprox/.
For \daps/ we replace the call to the diffusion sampling algorithm, and modify the outer noise schedule to match \cref{eq:sigma_sched2}. Algorithm details are provided in \cref{app:algos}. 

\textbf{Reducing computing costs.~~}
{Because of the double loop, replacing MMSE at every step of a PnP algorithm incurs a substantial computational overhead (in our case, with $K=4$ inner loop iterations, we multiply the compute cost by 4). We thus explored two different ways of reducing computing costs while keeping \trueprox/'s performance: using a hybrid schedule and distillation.} 
{\emph{(i) Distillation:}
A standard option is to get rid of the inner loop altogether, learning \trueprox/'s output in a single model. For this we finetune a copy of the same pretrained denoiser $D_\theta$ used in \trueprox/. We use a simple training process, optimizing an $\ell_1$ loss against \trueprox/'s outputs (computed with a frozen network $D_\theta$).}
{\emph{(ii) Hybrid schedule:}
Another way of reducing this cost while preserving performance comes from }the noise-matching analysis:  apply \trueprox/ only in the late iterations of the PnP outer loop, where the iterate's residual noise is low and the denoiser's score estimate is most reliable. In the early iterations, the iterate is still far from the data manifold, the score estimate is less critical, and a single MMSE call suffices.  {The two phases are  complementary: the MMSE steps produce an over-smoothed  estimate, and the sharpening effect of \trueprox/ (\cref{fig:denoising}) recovers the detail they removed. We show that it works particularly well for \dpir/, but has little effect on \diffpir/}.

%% file: sec/experiments.tex
\subsection{Validating the algorithm: standalone denoising}\label{sec:proximap-denoiser}

We begin by characterizing \trueprox/ as a standalone denoiser, both to compare it to other diffusion-based denoising algorithms, and to determine the iteration count $K$ to be used in the subsequent PnP experiments. 
Given $y = x + \sigma_y \epsilon$ with $\epsilon \sim \cN(0, I)$, the goal is to recover $x$. 
The denoiser $D_\theta$ is a diffusion model from \citet{dhariwal2021diffusion} trained on ILSVRC-2012 (ImageNet). 
We compare \trueprox/ to the conditional MMSE estimate $D_\theta(y, \sigma_y) \approx \bE[X \mid X_{\sigma_y} = y]$ and to three deterministic diffusion samplers adapted to the conditional setting: DDIM~\citep{song2021denoising}, Flow Matching~\citep{liu22flow,albergo22flow}, and InDI~\citep{delbracio23inversion}. Each sampler is initialized from the timestep corresponding to $y$ according to the original sampler's schedule. Performance is evaluated in terms of distortion (PSNR, corresponding to pixel-wise accuracy, higher is better) and perception (LPIPS, correlating with the human perception of sharpness, lower is better); Additional details are provided in \cref{app:denoising}.

\Cref{fig:denoising} shows two consistent patterns. First, all iterative methods improve LPIPS over the MMSE estimator while degrading PSNR. This is the standard perception-distortion trade-off~\citep{blau2018perception}. Second, \trueprox/ dominates the conditional samplers across all step counts and noise levels, sitting strictly closer to the optimal corner of the plane. LPIPS improvement saturates beyond 8 steps for \trueprox/ while computational cost grows linearly, which is confirmed by Figure~\ref{fig:ffhq-denoising} from~\cref{app:add_results} on the FFHQ dataset.

\begin{figure}
    \centering
    \includegraphics[width=\linewidth]{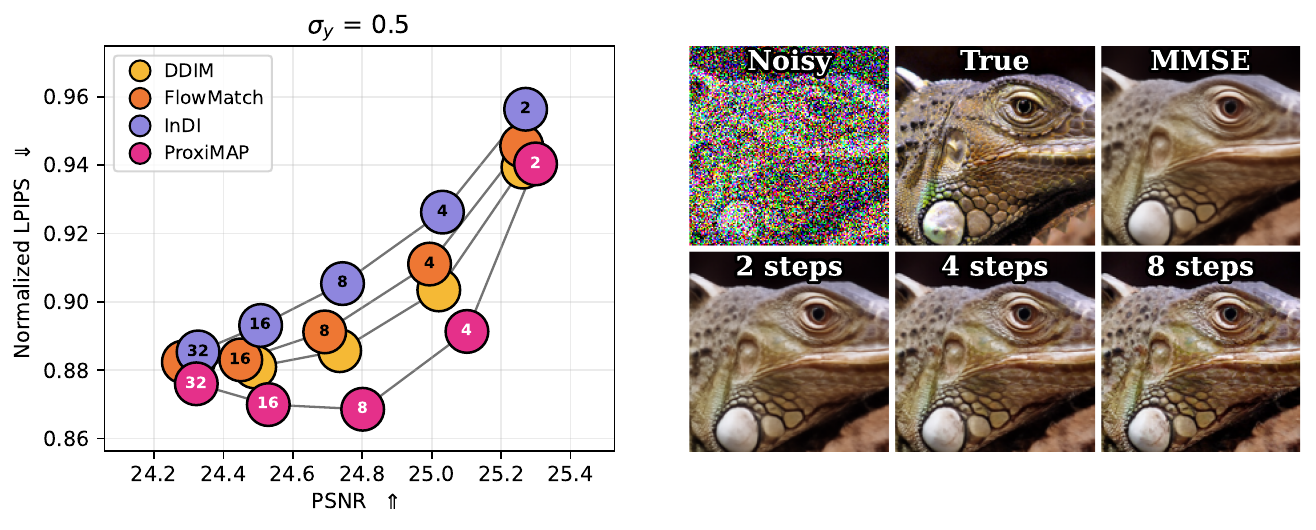}
    \caption{\emph{(left)} PSNR and LPIPS (normalized so the MMSE estimator's LPIPS is 1) across step counts (all algorithms were run for 2, 4, 8, 16 and 32 steps) for \trueprox/ and three conditional samplers on ImageNet with an initial noise level of $\sigma=0.5$. \emph{(right)} Outputs of \trueprox/ for increasing $K$. Results on FFHQ are in \cref{fig:ffhq-denoising}.}
    \label{fig:denoising}
\end{figure}

\begin{figure}[!ht]
    \centering
    \includegraphics[width=0.90\linewidth]{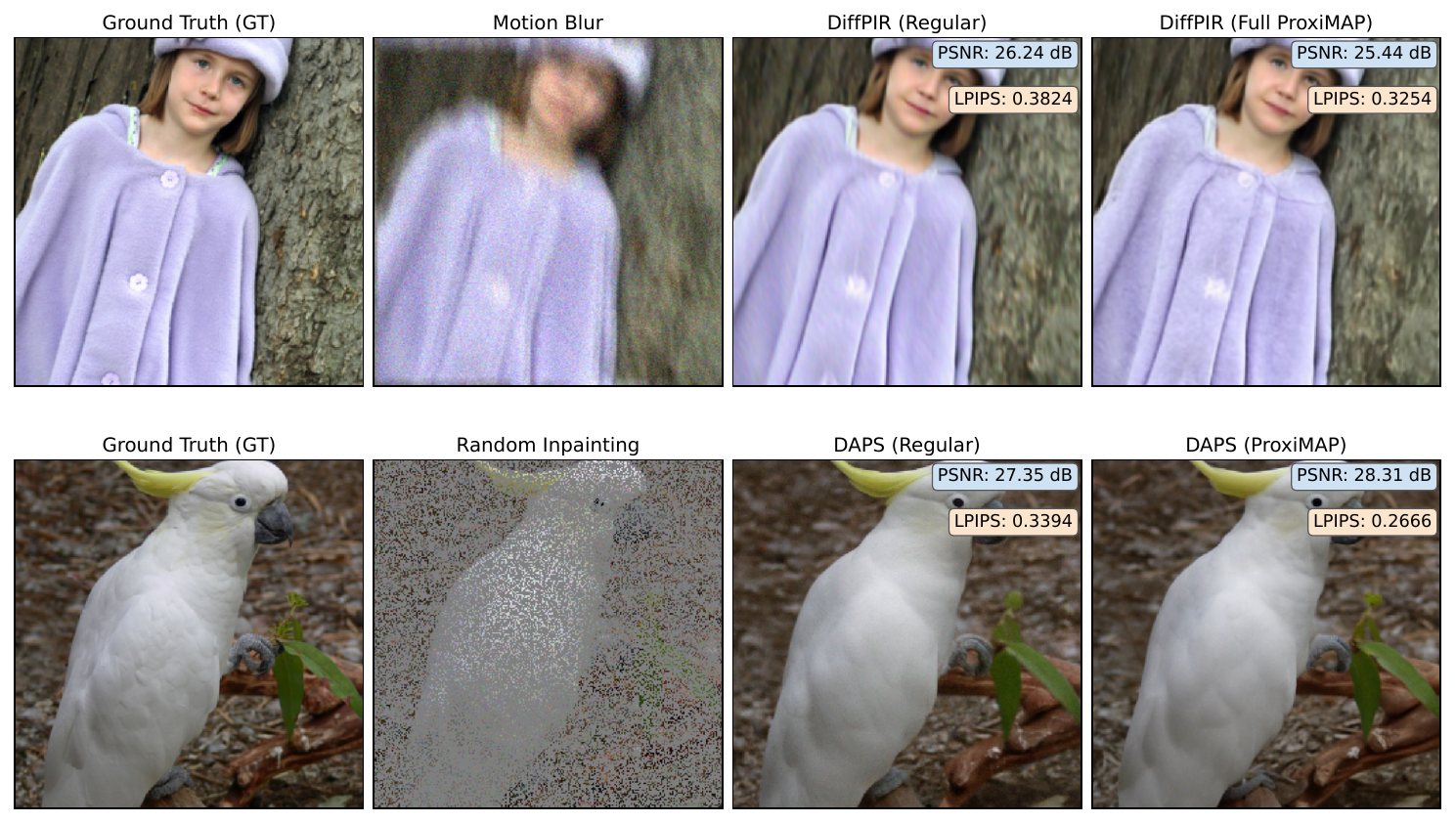}
    \caption{ImageNet examples: \diffpir/ on motion blur (top) and \daps/ on random inpainting (bottom), baseline vs.\ \trueprox/. \trueprox/ restores the finer details already seen in the denoising task of \cref{fig:denoising}. More examples are in \cref{app:add_results}.}
    \label{fig:pnp_img_ex}
\end{figure}

\begin{figure}[!ht]
    \centering
    \includegraphics[width=0.98\linewidth]{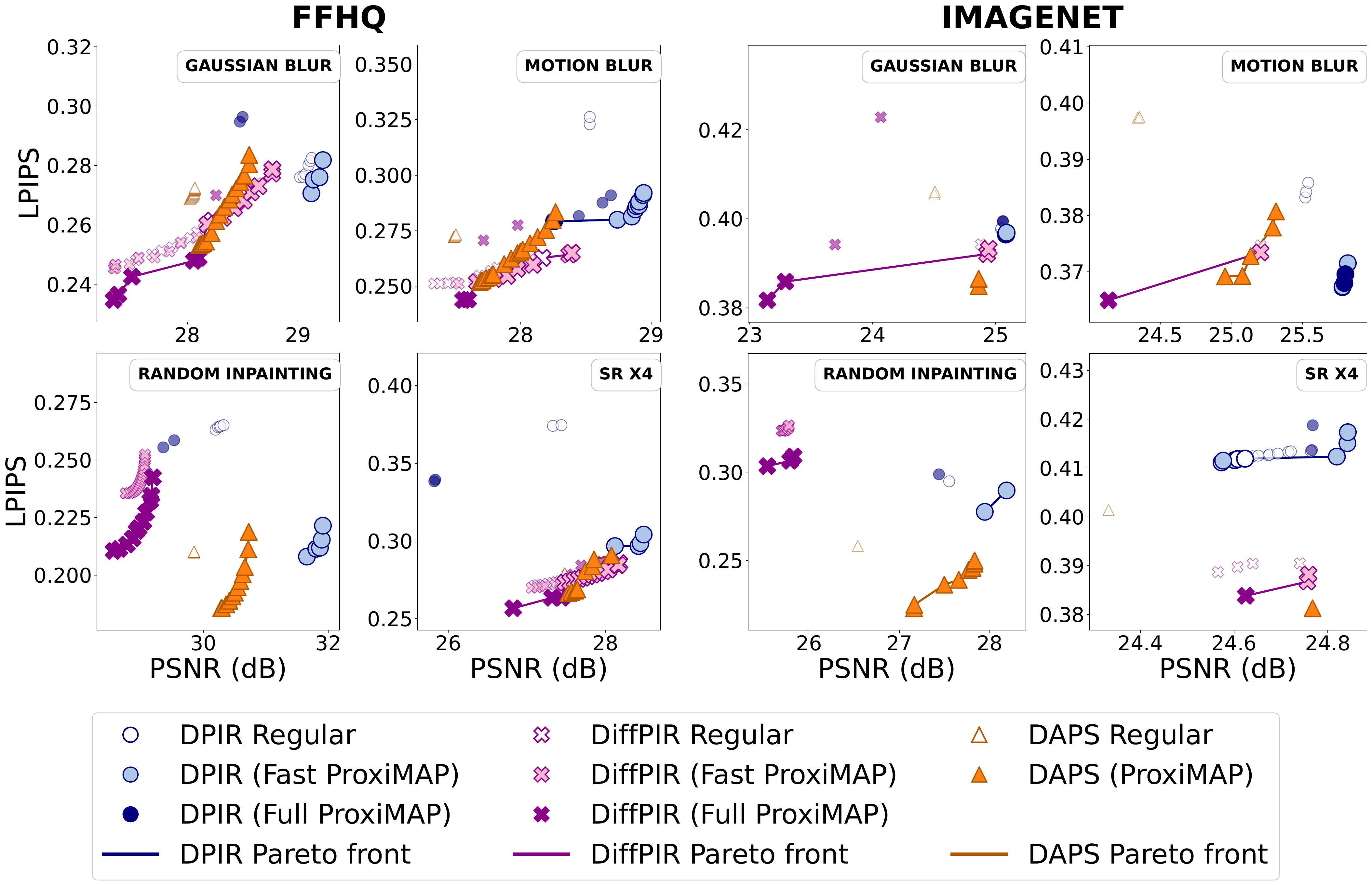}
    \caption{Test-set Pareto fronts of three PnP algorithms and their \trueprox/ variants on four inverse problems (FFHQ left, ImageNet right, $\sigma_y = 0.05$). Markers indicate the algorithm; hollow markers are the baselines, light-filled markers \texttt{Fast} \trueprox/ and dark-filled markers full \trueprox/ (\daps/ has no \texttt{Fast} variant). For each algorithm, the line is the Pareto front of all its variants taken together; points off this front are smaller and faded. These fronts are formed almost entirely by \trueprox/ variants, which extend them toward lower LPIPS, sometimes at the cost of PSNR (e.g., \diffpir/ on the ImageNet blurs).}
    \label{fig:fp-comp}
\end{figure}

\subsection{PnP experiments}\label{sec:pnp-experiments}

We next evaluate \trueprox/ as a drop-in replacement for the MMSE denoiser in the three state-of-the-art PnP frameworks we introduced in \cref{sec:proximap-pnp}: \dpir/, \diffpir/, and \daps/. All experiments use FFHQ~\citep{karras2019style} and ImageNet~\citep{imagenet} samples at $256 \times 256$ resolution; the underlying diffusion models are fixed (the network of~\citet{choi21ilvr} for FFHQ and the unconditional model of~\citet{dhariwal2021diffusion} for ImageNet). {For \dpir/, we use 8 outer iterations for deblurring, as in the original paper \citep{zhang2021plug}, and 20 for random inpainting and super-resolution. \diffpir/ uses 20 outer iterations for all tasks. For both methods, we fix this number of outer iterations and replace each MMSE call (one every iteration) by \trueprox/.} \daps/'s inner loop already runs an iterative sampling procedure, which we replace by 6 \trueprox/ iterations at no additional cost. We consider several tasks for the evaluation, which are detailed in \cref{app:tasks}.

\paragraph{Automatic hyperparameter tuning protocol.} PnP methods are sensitive to hyperparameter choices, which are usually changed depending on the task and the noise level. However, reporting is often inconsistent in the literature. For example, certain papers~\citep{diffpir,reddiff} scale images to $[0,1]$ while others~\citep{pnpflow,sitcom} use $[-1,1]$, leading to noise levels that do not match. To allocate the same tuning budget to every method, we run Bayesian optimization (BO) for each \emph{(dataset, algorithm, problem)} triplet using the Ax library~\citep{olson2025ax} with {30} random-search iterations followed by {30} BO iterations, optimizing both LPIPS and PSNR {(with outcome constraints PSNR $\geq 15$\,dB and LPIPS $\leq 0.5$ to exclude degenerate configurations)} on a held-out tuning set of 30 images. Using the hyperparameters lying on the Pareto front which results from BO, we then compute all metrics on a separate test-set of 70 images.  Full hyperparameter ranges and tuning details for every method are provided in \cref{app:pnp}.

\paragraph{Drop-in replacement.} The full-replacement {variants of MMSE by \trueprox/ in PnP algorithms} are compared in \cref{tab:big-comp} alongside four further diffusion-based PnP algorithms (\dps/~\citep{dps}, \sitcom/~\citep{sitcom}, \reddiff/~\citep{reddiff}, and \pnpflow/~\citep{pnpflow}). Qualitative examples are shown in \cref{fig:pnp_img_ex}.
\trueprox/ improves the LPIPS of its base algorithm in 19 of the 24 (algorithm, task, dataset) settings and is tied in 4. Its effect on PSNR depends on the base algorithm: \daps/+\trueprox/ is never worse than \daps/ except on HDR where it is tied, whereas \dpir/ and \diffpir/ trade PSNR for perceptual quality. \Cref{fig:fp-comp} shows the corresponding Pareto fronts for the four linear tasks with $\sigma_y = 0.05$ (see also \cref{fig:full-replacement_comparison} in \cref{app:add_results}). The largest PSNR losses are \dpir/+\trueprox/ on FFHQ (up to 1.6\,dB on super-resolution) and \diffpir/+\trueprox/ on the ImageNet blurs (1.1 and 1.8\,dB). For \dpir/, the outer schedule does not track the residual noise of its iterates, so \trueprox/ is called at noise levels that are not the ones actually present and the denoiser is queried off-distribution; this is the failure mode the hybrid variant avoids (\cref{app:failure}). No single method dominates every task, which reflects both the perception-distortion trade-off~\citep{blau2018perception} and the heterogeneity of the problems. The \trueprox/-augmented methods are competitive with the state of the art: rather than a new specialized algorithm, our contribution is a modular component that improves the perceptual quality of an entire family of existing methods.

\paragraph{Reducing the cost}
Replacing every MMSE call with \trueprox/ multiplies the NFEs of \dpir/ and \diffpir/ by $K=4$. We evaluate the two cheaper strategies of \cref{sec:proximap-pnp} in \cref{tab:reduce_compute_ffhq,tab:reduce_compute_imagenet}. For the hybrid variant, we sweep $n$, the number of final outer iterations that use \trueprox/, and find that $n=1$ works best for \dpir/. We call this variant \texttt{Fast} \trueprox/: at 3 extra NFEs, it outperforms both the baseline and  the full replacement variant in most cases. For \diffpir/, it has no effect. \diffpir/'s last iteration runs at $\sigma=0.005$, which can be smaller than the tuned final noise level $\sigma_{end}$ of the inner loop, meaning that the MMSE is called instead of \trueprox/ in the last iteration. The distilled student replaces the MMSE denoiser at every iteration at no extra cost. For \diffpir/, it recovers most of the LPIPS gain of full replacement (within 0.005 on FFHQ), but loses 0.9--2.2\,dB PSNR on ImageNet; for \dpir/, its results are mixed. More details are provided in \cref{app:reduce_compute}.

\begin{table}[t]
      \caption{ImageNet results, $\sigma_y = 0.05$, at the lowest-LPIPS sample of each method's Pareto front (FFHQ: \cref{tab:ffhq-comp}; SSIM: \cref{tab:ssim}). Empty cells: \dpir/ and \diffpir/ need a closed-form proximal step, so they do not apply to non-linear tasks. NFEs are doubled for methods requiring backward passes. \texttt{Fast} \trueprox/ applies \trueprox/ ($K=4$) at the last outer iteration only; \dpir/ runs 8 outer iterations on the blurs and 20 on SR and inpainting. {\color{blue}Blue}: LPIPS at least $0.005$ below the baseline.
  \vspace*{0.2cm}}
    \label{tab:big-comp}
    \setlength{\tabcolsep}{2pt}
    \resizebox{\linewidth}{!}{
    \begin{tabular}{lcccccccccccccc}
    \toprule
    \multirow{2}{*}{Algorithm} & 
    \multicolumn{2}{c}{Motion Blur} & \multicolumn{2}{c}{Gaussian Blur} & 
    \multicolumn{2}{c}{SR $\times 4$} & \multicolumn{2}{c}{Inpainting} & 
    \multicolumn{2}{c}{HDR} & \multicolumn{2}{c}{Phase Retrieval} & 
    \multirow{2}{*}{NFE} \\
    \cmidrule(lr){2-3} \cmidrule(lr){4-5} \cmidrule(lr){6-7} \cmidrule(lr){8-9} \cmidrule(lr){10-11} \cmidrule(lr){12-13}
    & PSNR & lpips & PSNR & lpips & PSNR & lpips & PSNR & lpips & PSNR & lpips & PSNR & lpips & \\
    \midrule
    
    \dpir/                       & 25.5 & 0.383 & 25.0 & 0.398 & 24.6 & 0.412 & 27.6 & 0.295 & -- & --  & -- & -- & 8\,/\,20 \\
    ~~+ \trueprox/               & 25.8 & \textcolor{blue}{0.368} & 25.1 & 0.399 & 24.8 & 0.414 & 27.4 & 0.299 & -- & --  & -- & -- & 32\,/\,80 \\
    ~~+ \texttt{Fast} \trueprox/ & 25.8 & \textcolor{blue}{0.367} & 25.1 & 0.396 & 24.6 & 0.411 & 27.9 & \textcolor{blue}{0.278} & -- & -- & -- & -- & 11\,/\,23 \\
    \hdashline\vspace*{-0.25cm}\\ 
    \diffpir/                    & 25.2 & 0.375 & 24.9 & 0.394 & 24.6 & 0.389 & 25.8 & 0.324 & -- & --  & -- & -- & 20 \\
    ~~+ \trueprox/               & 24.1 & \textcolor{blue}{0.365} & 23.1 & \textcolor{blue}{0.382} & 24.6 & 0.384 & 25.5 & \textcolor{blue}{0.303} & -- & --  & -- & -- & 80 \\
    ~~+ \texttt{Fast} \trueprox/ & 25.2 & 0.373 & 24.9 & 0.392 & 24.8 & 0.387 & 25.8 & 0.324 & -- & -- & -- & -- & 23 \\
    \hdashline\vspace*{-0.25cm}\\ 
    \daps/                       & 24.4 & 0.397 & 24.5 & 0.405 & 24.3 & 0.401 & 26.5 & 0.258 & 19.6 & 0.317 & 15.9 & 0.560 & 1200 \\
    ~~+ \trueprox/               & 25.1 & \textcolor{blue}{0.369} & 24.9 & \textcolor{blue}{0.385} & 24.8 & \textcolor{blue}{0.381} & 27.2 & \textcolor{blue}{0.223} & 19.5 & \textcolor{blue}{0.310} & 16.3 & \textcolor{blue}{0.536} & 1200 \\
    \hdashline\vspace*{-0.25cm}\\
    \pnpflow/                    & 23.6 & 0.468 & 24.3 & 0.434 & 24.3 & 0.438 & 28.3 & 0.235 & 18.0 & 0.337 & 14.8 & 0.596 & 500 \\
    \reddiff/                    & 23.8 & 0.463 & 24.4 & 0.434 & 24.4 & 0.430 & 27.2 & 0.262 & 18.2 & 0.326 & 14.9 & 0.648 & 1000 \\
    \sitcom/                     & 25.0 & 0.335 & 25.3 & 0.353 & 25.0 & 0.365 & 26.1 & 0.313 & 19.2 & 0.350 & 14.0 & 0.612 & 800 \\
    \dps/                        & 23.7 & 0.357 & 23.8 & 0.338 & 24.3 & 0.370 & 27.0 & 0.236 & 19.8 & 0.272 & 12.6 & 0.642 & 2000 \\
    \bottomrule
\end{tabular}}
 \end{table}

%% file: sec/appendix.tex
\section{A Discussion on MAP Estimation}\label{app:discussion}
\input{sec/theoryvpractice.tex}

\section{Memorization in Diffusion Models}\label{app:memorizing}

We train two flow-matching models on the CIFAR10 dataset. The first -- the \emph{generalizing model} -- is trained on the full training set of 50000 images, the second -- the \emph{memorizing model} -- is trained on only 5000 images. Both models have the same number of parameters.
We observe that the two training losses follow a similar trajectory in the early stages of training, until a certain point after which the training loss of the memorizing model rapidly decreases.
The model thus enters the \emph{memorizing regime}~\citep{biroli2024dynamical} in which it is unable to generate novel samples and will replicate ones from the training set.
We test both models on two smoothed GD schedules: the theoretically supported one of \cref{eq:original-algo} (proposed by~\citet{pesme2025map}) and the \emph{early-stopped} one proposed here. 
Referencing \cref{fig:memorizing-cifar10}, for the generalizing model, we observe a similar effect to what we reported in \cref{app:discussion}: the MAP estimation needs to be early-stopped to prevent ending into the \emph{cartoon region}.
This however does not happen for the memorizing model. As the flow-matching network learns to replicate its training samples its score becomes more accurate and the fast, non-early-stopped algorithm of \citet{pesme2025map} generates clean, realistic images.
Of course, this does not mean that we should use a memorizing model: while it generates realistic images, these will always come from the training set - rendering it useless for any task which involves unseen samples.

\begin{figure}
    \centering
    \includegraphics[width=0.95\linewidth]{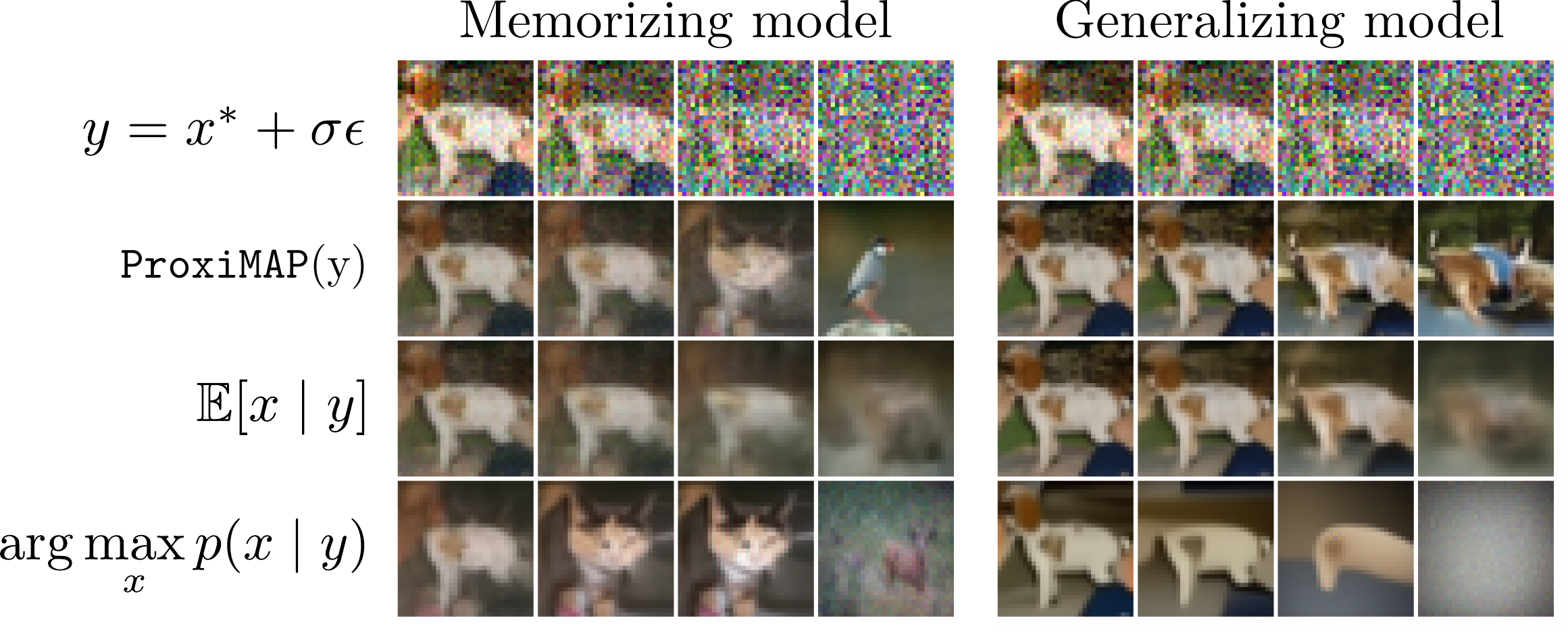}
    \caption{Comparing the conditional performance of a memorizing and a generalizing diffusion model. The rows correspond to a) the conditioning variable (which shows increasing amounts of noise from left to right), b) images obtained with \trueprox/, c) the MMSE estimate and d) the MAP obtained via \cref{eq:original-algo}.}
    \label{fig:memorizing-cifar10}
\end{figure}

\section{Behaviour of the Recursive Noise Sequence \texorpdfstring{$(\sigma_k)_k$}{(σk)k}}\label{app:sec:proof-lemma}

Consider the recursive sequence defined by
\begin{equation}\label{eq:rec}
    \sigma_{k+1} = (1 - \beta)\,\sigma_k
                 + \frac{\beta\,\sigma_0}{1 + \dfrac{\tau}{\sigma_k^2}}
\end{equation}
with parameters $\beta \in (0,1)$, $\tau > 0$, $\sigma_0 > 0$, and
initialisation $\sigma_{k=0} = \sigma_0$.

\begin{lemma}\label{app:lemma}
Let $(\sigma_k)_{k \geq 0}$ be the sequence defined by~\eqref{eq:rec}.
Then $\sigma_k \to 0$ as $k \to \infty$ if and only if
\[
    \tau > \frac{\sigma_0^2}{4}.
\]
\end{lemma}

\begin{proof}
Define the map $f : (0, \infty) \to (0, \infty)$ by
\[
    f(\sigma) = (1-\beta)\,\sigma + \frac{\beta\,\sigma_0}{1 + \tau/\sigma^2}
              = (1-\beta)\,\sigma + \frac{\beta\,\sigma_0\,\sigma^2}{\sigma^2 + \tau},
\]
such that $\sigma_{k+1} = f(\sigma_k)$.
Note that $f(\sigma) > 0$ for all $\sigma > 0$, so the sequence stays positive.
Moreover, $f$ extends continuously to $\sigma = 0$ with $f(0) = 0$, so $\sigma = 0$ is always a fixed point. Concerning potential other fixed points, notice that a positive fixed point satisfies $f(\sigma) = \sigma$, i.e.\
\[
    (1-\beta)\,\sigma + \frac{\beta\,\sigma_0\,\sigma^2}{\sigma^2+\tau} = \sigma.
\]
Subtracting $(1-\beta)\sigma$ from both sides and dividing by $\beta\sigma > 0$ gives
\[
    \frac{\sigma_0\,\sigma}{\sigma^2+\tau} = 1
    \implies \sigma^2 - \sigma_0\,\sigma + \tau = 0.
\]
The discriminant of this quadratic is $\sigma_0^2 - 4\tau$ and real solutions exist if and only if $\tau \leq \sigma_0^2/4$ and the fixed points are then
\[   \sigma_\pm = \frac{\sigma_0 \pm \sqrt{\sigma_0^2 - 4\tau}}{2} > 0.\]
\noindent
We now split the proof into the two directions.  

\noindent\textbf{Sufficient condition: $\tau > \sigma_0^2/4 \implies \sigma_k \to 0$.}
\noindent
When $\tau > \sigma_0^2/4$, then $f$ has only $0$ as a fixed point. We now show $f(\sigma) < \sigma$ for all $\sigma > 0$, so $(\sigma_k)$ is strictly decreasing.
\begin{align*}
    f(\sigma) < \sigma
    &\iff (1-\beta)\sigma + \frac{\beta\sigma_0\sigma^2}{\sigma^2+\tau} < \sigma \\
    &\iff \frac{\beta\sigma_0\sigma^2}{\sigma^2+\tau} < \beta\sigma \\
    &\iff \frac{\sigma_0\sigma}{\sigma^2+\tau} < 1 \\
    &\iff \sigma^2 - \sigma_0\sigma + \tau > 0.
\end{align*}
The last inequality is true for all $\sigma$ if $\tau > \sigma_0^2/4$. Hence $(\sigma_k)$ is strictly decreasing and bounded below by $0$, so it converges to some limit $\ell \geq 0$. Taking $k\to\infty$ in the recursion $\sigma_{k+1} = f(\sigma_k)$ yields $\ell = f(\ell)$, so $\ell$ is a fixed point of $f$ in $[0,\infty)$. Since the only such fixed point is $0$, we conclude $\sigma_k \to 0$.

\medskip
\noindent\textbf{Necessary condition: $\tau \leq \sigma_0^2/4$, then $\sigma_k \not\to 0$.}
\noindent
If $\tau \leq \sigma_0^2/4$, then the largest fixed point $\sigma_+ = \frac{\sigma_0 + \sqrt{\sigma_0^2 - 4\tau}}{2} > 0$ is strictly smaller than $\sigma_0$, which means that the initialisation $\sigma_{k=0}$ is larger than any fixed point.
Now recall that 
\[f(\sigma) < \sigma \iff \sigma^2 - \sigma_0\sigma + \tau > 0.\]
Which means that $f(\sigma) < \sigma$ for $\sigma > \sigma_+$. Also a simple analysis of the function $f$ shows that it is increasing over $\mathbb{R}_{\geq 0}$. Putting things together, if $\sigma_+ \leq \sigma_k$, then 
\[\sigma_+ = f(\sigma_+) \leq \sigma_{k+1} = f(\sigma_k) \leq \sigma_k. \]
Since $\sigma_{k=0} = \sigma_0 > \sigma_+$, we have that $\sigma_k$ is a decreasing function lower-bounded by $\sigma_+$. It follows that it must converge towards some limit $\ell \geq \sigma_+$ such that $f(\ell) = \ell$. Hence $\ell = \sigma_+ \neq 0$ which concludes the proof.
\end{proof}

\section{Details about PnP Algorithms}\label{app:algos}

\Cref{algo:dpir-full,algo:diffpir-full,algo:daps-full} provide more precise pseudo-code of the algorithms we used in this work. For \dpir/ and \diffpir/, we simply plugged \trueprox/ in-place of the denoiser $D_\theta$. For \daps/ we also changed the outer-loop noise schedule (which in the original paper is the EDM~\citep{edm} schedule) to follow the \trueprox/ schedule. The resulting iterative algorithm is presented in~\Cref{algo:daps-pm-full}.

\begin{algorithm}
    \caption{\dpir/}\label{algo:dpir-full}
    \begin{algorithmic}[1]
        \Require data-fidelity $f$, denoiser $D_\theta$, $y$, $\sigma_y$, $N$
        \Require hyperparameters $\sigma_{\max}$, $\gamma$
        \State $(\sigma_k)_{k=0}^{N - 1} \gets \mathrm{logspace}(\log(\sigma_{\max}), \log(\sigma_y), N)$
        \State $x_0 \gets \cA^\top y$
        \For{$k = 0, \dots, N - 1$}
            \State $\gamma_k \gets \gamma \left(\frac{\sigma_k}{\sigma_y}\right)^2$
            \State $z \gets \prox_{\gamma_{k} f}(x_{k})$
            \State $x_{k + 1} \gets D_\theta(z, \sigma_{k})$ \hfill\Comment{Replace this with a call to \trueprox/}
        \EndFor
        \State \Return $x_{N}$
    \end{algorithmic}
\end{algorithm}

\begin{algorithm}
    \caption{\diffpir/}\label{algo:diffpir-full}
    \begin{algorithmic}[1]
    \Require data-fidelity $f$, denoiser $D_\theta$, $y$, $\sigma_y$, $N$
    \Require hyperparameters $\zeta$, $\lambda$
    \Require DDIM schedule $(\bar{\alpha}_k)_{k=0}^{N}$
    \State $x_0 \sim \cN(0, I)$ \hfill\Comment{More refined initialization in practice}
    \For{$k=0, \dots, N - 1$}
        \State $\rho_k \gets \lambda \sigma_y^2 \left(\frac{\sqrt{\bar{\alpha}_k}}{\sqrt{1-\bar{\alpha}_k}}\right)^2$

        \State $\hat{x}_{\mathrm{MMSE}} \gets D_\theta\left(\frac{x_k}{\sqrt{\bar{\alpha}_k}}, \frac{\sqrt{1-\bar{\alpha}_k}}{\sqrt{\bar{\alpha}_k}}\right)$\hfill\Comment{Replace this with a call to \trueprox/}
        \State $\hat{x} \gets \prox_{\frac{1}{\rho_k} f}(\hat{x}_{\mathrm{MMSE}})$

        \State $\hat{\epsilon} \gets \frac{x_k - \sqrt{\bar{\alpha}_k}\,\hat{x}}{\sqrt{1-\bar{\alpha}_k}}$

        \State $\epsilon \sim \cN(0, I)$
        \State $x_{k + 1} \gets \sqrt{\bar{\alpha}_{k+1}}\, \hat{x} + \sqrt{1-\bar{\alpha}_{k+1}} \left(\sqrt{1 - \zeta}\,\hat{\epsilon} + \sqrt{\zeta}\,\epsilon\right)$
    \EndFor
    \State \Return $\hat{x}_{\mathrm{MMSE}}$
    \end{algorithmic}
\end{algorithm}

\begin{algorithm}
    \caption{\daps/}\label{algo:daps-full}
    \begin{algorithmic}[1]
        \Require data-fidelity $f$, denoiser $D_\theta$, $y$, $\sigma_y$, $N$, $J$
        \Require hyperparameters $\gamma_{\mathrm{init}}$
        \Require EDM schedule $(\sigma_k)_{k=0}^{N}$
        \State $\gamma_{\min} \gets 0.01$
        \State $x_0 \sim \cN(0, \sigma_{\max}^2 I)$
        \For{$k = 0, \dots, N-1$}
            \State $\hat{x}_{\mathrm{sample}} \sim p_0(X\mid X_{\sigma_k} = x_k, D_\theta)$ \hfill\Comment{{\tiny PF-ODE sampling (EDM schedule, Euler solver)}}
            \State $z_0 \gets \hat{x}_{\mathrm{sample}}$
            \State $\gamma_{k} \gets \gamma_{\mathrm{init}} \left(1 + \frac{k}{N} (\gamma_{\min} - 1)\right) $
            \For{$j=0, \dots, J-1$} \hfill\Comment{Langevin sampling}
                \State $\epsilon \sim \cN(0, I)$
                \State $z_{j+1} \gets z_j - \gamma_k \left( \nabla f(z_{j}, y) + \frac{z_{j} - z_0}{\sigma_k^2}\right) + \sqrt{2\gamma_k}\epsilon$
            \EndFor
            \State $x_{k+1} \gets z_J + \cN(0, \sigma_{k+1}^2 I)$
        \EndFor
        \State \Return $z_J$
    \end{algorithmic}
\end{algorithm}

\begin{algorithm}
    \caption{\daps/+\trueprox/}\label{algo:daps-pm-full}
    \begin{algorithmic}[1]
        \Require data-fidelity $f$, denoiser $D_\theta$, $y$, $\sigma_y$, $N$, $J$
        \Require hyperparameters $\gamma_{\mathrm{init}}$, $\sigma_{\mathrm{final}}$
        \Require \trueprox/ schedule $(\sigma_k)_{k=0}^{N}$
        \State $\gamma_{\min} \gets 0.01$
        \State $x_0 \sim \cN(0, \sigma_{\max}^2 I)$
        \For{$k = 0, \dots, N-1$}
            \State $\hat{x}_{\trueprox/} \gets \trueprox/(x_k, \sigma_k, D_\theta)$
            \State $z_0 \gets \hat{x}_{\trueprox/}$
            \State $\gamma_{k} \gets \gamma_{\mathrm{init}} \left(1 + \frac{k}{N} (\gamma_{\min} - 1)\right) $
            \For{$j=0, \dots, J-1$} \hfill\Comment{Langevin sampling}
                \State $\epsilon \sim \cN(0, I)$
                \State $z_{j+1} \gets z_j - \gamma_k \left( \nabla f(z_{j}, y) + \frac{z_{j} - z_0}{\sigma_k^2}\right) + \sqrt{2\gamma_k}\epsilon$
            \EndFor
            \State $x_{k+1} \gets z_J + \cN(0, \sigma_{k+1}^2 I)$
        \EndFor
        \State \Return $z_J$
    \end{algorithmic}
\end{algorithm}

\section{Denoising Experimental Details}\label{app:denoising}

For the denoising experiment of \cref{fig:denoising}, we used the full testing datasets at our disposal (100 images for FFHQ, 100 for ImageNet). The task was simply to remove additive Gaussian noise with increasing variance from each test image. We set $\sigma_{\mathrm{final}}$ to 0.005 for \trueprox/ for all runs. The other algorithms (InDI, DDIM, Flow-matching) were similarly not tuned further. 
We used PSNR and LPIPS as metrics to highlight the perception-distortion tradeoff. In particular -- throughout the paper -- LPIPS was computed with the VGG-16 network.
In the main text we showed the ImageNet results. In \cref{fig:imagenet-denoising,fig:ffhq-denoising} we show the results on FFHQ. The conclusions are the same between the two datasets: while all compared algorithms perform similarly, \trueprox/ converges with fewer iterations to good solutions dominating the competing methods.

\begin{figure}
    \centering
    \includegraphics[width=0.85\linewidth]{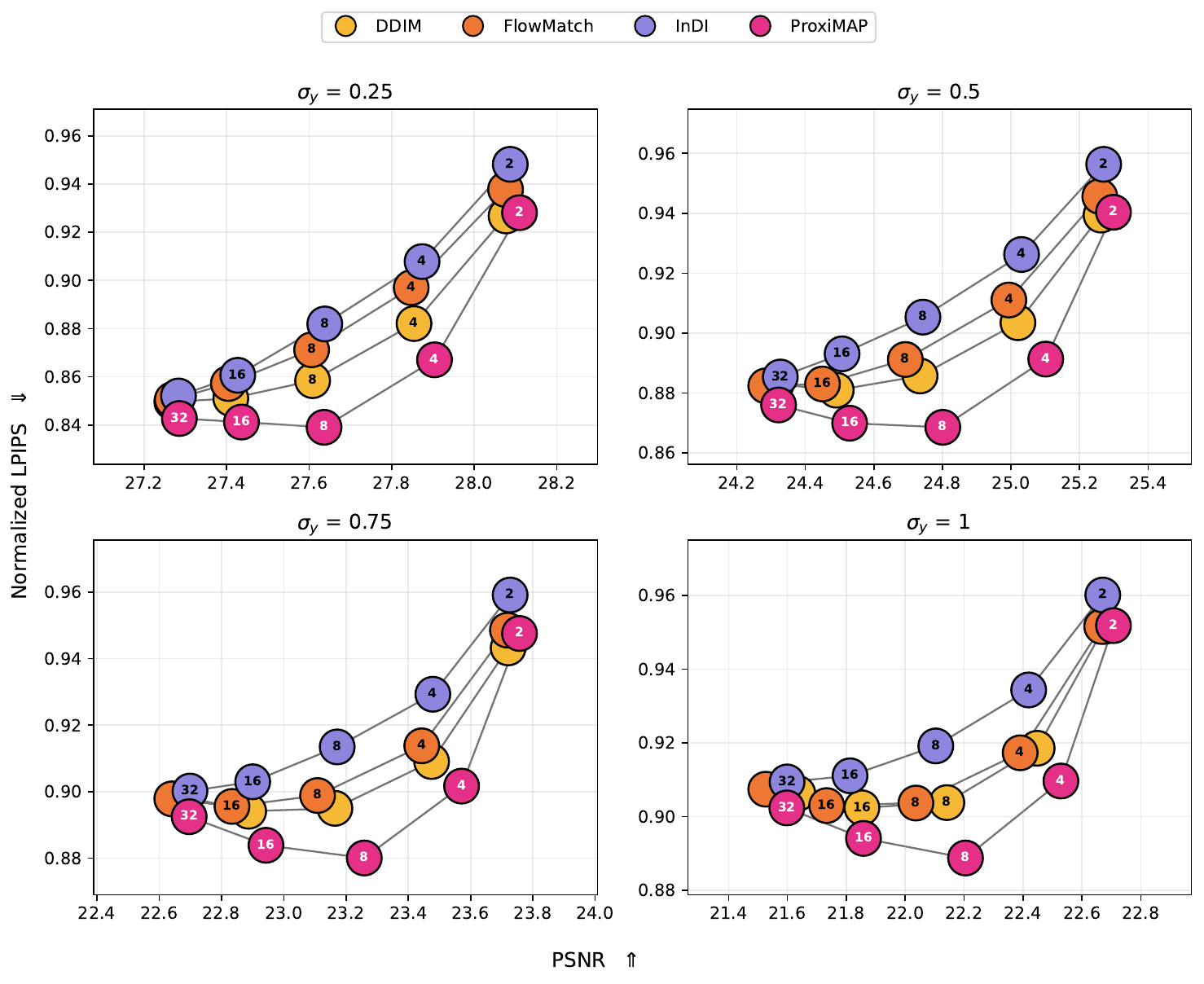}
    \caption{ImageNet Gaussian denoising results for various noise levels. LPIPS is normalized such that the MMSE denoiser has an LPIPS of 1. Each algorithm is run for 2-4-8-16-32 steps.}
    \label{fig:imagenet-denoising}
\end{figure}

\begin{figure}
    \centering
    \includegraphics[width=0.85\linewidth]{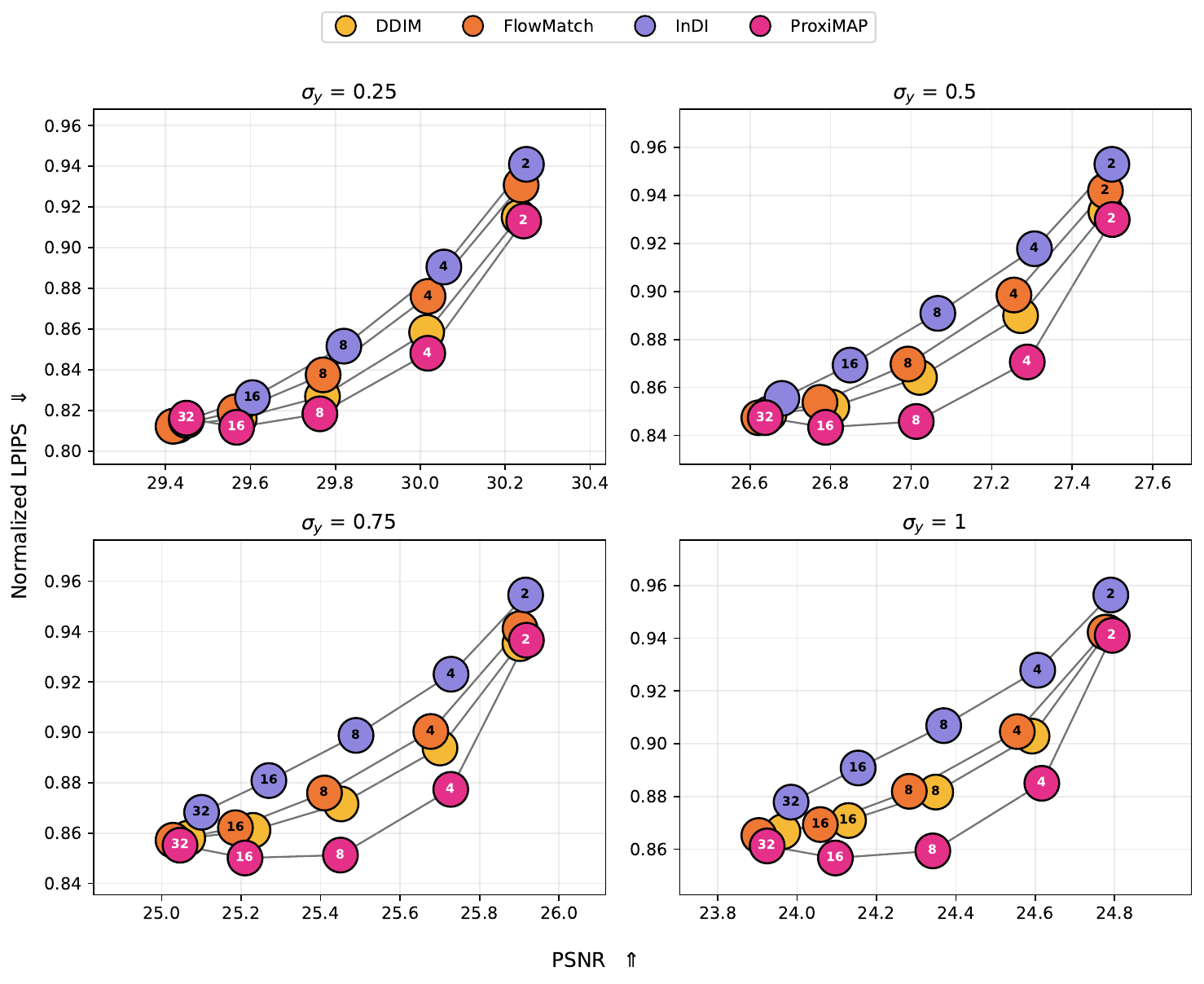}
    \caption{FFHQ Gaussian denoising results for various noise levels. LPIPS is normalized such that the MMSE denoiser has an LPIPS of 1. {Each algorithm is run for 2-4-8-16-32 steps}}
    \label{fig:ffhq-denoising}
\end{figure}

\section{Description of the Tasks}\label{app:tasks}

We describe in more detail the inverse problems on which we ran evaluations in the paper. The first 5 are linear problems, for which the $\prox$ of the data-fidelity term can be computed easily. 
The last two instead are non-linear and do not have a closed-form proximal operator. For this reason we did not test \dpir/ and \diffpir/ (which require the proximal operator of the data-fidelity term) on the non-linear tasks.
All tasks  were run with levels $\sigma=0.05$ with additive Gaussian noise applied to the corrupted images in the $[0, 1]$ range. {We used the library \texttt{deepinv} \citep{tachella2025deepinverse} to implement the most of the operators.}
\begin{itemize}
    \item \textbf{Gaussian blur.}
    We used a kernel of size 61x61, with a standard deviation of 3 and \emph{circular} padding.
    \item \textbf{Motion blur.}
    We used a standard \href{https://github.com/LeviBorodenko/motionblur}{kernel generating function} with kernel size 61x61 and intensity of 0.5. Once again with circular padding.
    \item \textbf{Superresolution.} $\times 4$ decimation preceded by a bicubic filter, circular padding.
    \item \textbf{Random inpainting.} Masked 70\% of pixels setting them to the value 0.5.
    \item \textbf{High dynamic range.} A non-linear forward model consisting of multiplying all pixel values by 2 and clamping them to lie within the $[0, 1]$ range.
    \item \textbf{Phase retrieval.} {A second non-linear task, $y = |\mathcal{F}(Px)| + n$. Each channel of the image is zero-padded by $64$ pixels on every side ($P$, from $256 \times 256$ to $384 \times 384$), and only the magnitude of its centred, orthonormal 2D Fourier transform $\mathcal{F}$ is measured; the noise $n$ is added to these magnitudes. The model is invariant to translations of each channel within the padded frame and to its $180^\circ$ rotation $x(t) \mapsto x(-t)$, so a reconstruction may be correct only up to these ambiguities. We compute metrics against the ground truth without compensating for them. The solvers use the data term only through its gradient, computed by automatic differentiation.}
\end{itemize}

\section{PnP Experiment Details}\label{app:pnp}

To run the plug-and-play experiment, for which results are reported in \cref{fig:fp-comp,tab:big-comp}, we re-implemented and ran the competing algorithms in order to allocate fair amounts of hyperparameter tuning effort. The hyperparameters and their tuning ranges are described below. Note that the hyperparameters ranges for \trueprox/ and for Fast \trueprox/ were kept the same.
\begin{itemize}
    \item \dpir/. We fixed the number of iterations to 8 for the deblurring tasks and 20 for super-resolution and inpainting and performed hyperparameter tuning by sweeping over two additional control knobs. The first one controls the schedule of decreasing noise levels with which the denoiser is called. The maximum noise level was swept between $[0.5, 78]$ on a logarithmic scale. The second controls the weight given to the data-fidelity term. It was originally set to approximately $5$ and we swept it over the range $[2, 30]$ for deblurring and $[2, 100]$ for super-resolution and inpainting, on a logarithmic scale.
    \item \dpir/+\trueprox/. We fixed the number of \trueprox/ iterations to 4 and additionally swept the $\sigma_{\mathrm{final}}$ parameter between $[10^{-4}, 0.1]$ on a logarithmic scale, keeping the \dpir/ ranges.
    \item \diffpir/. We fixed the number of iterations to 20 since -- while not the best performance possible -- it reaches fairly good results with minimal computation time. We swept $\lambda$ over $[1, 316]$ on a logarithmic scale and $\zeta$ over $[0.5, 1]$.
    \item \diffpir/+\trueprox/. We fixed the number of \trueprox/ iterations to 4 and additionally swept the $\sigma_{\mathrm{final}}$ parameter between $[10^{-4}, 0.1]$ on a logarithmic scale, keeping the \diffpir/ ranges.
    \item \daps/. We fixed the number of Langevin steps (data-fidelity optimization) to 50, the number of outer iterations to 200 and the number of ODE sampling (inner) iterations to 6. We swept over the Langevin learning-rate parameter (which controls the strength of the data-fidelity term) between $[0.000001, 0.0002]$. 
    \item \daps/+\trueprox/. In addition to \daps/ we swept over the $\sigma_{\mathrm{final}}$ parameter for the inner \trueprox/ iteration between $[0.001, 0.2]$. We fixed the outer \trueprox/ schedule to have $\sigma_{\mathrm{final}}=0.1$ which is close to the final noise-level in DDIM for 200 steps. Both \trueprox/ schedules use $\tau = 2.5\,\sigma^2$ (instead of $\tau = \sigma^2$ for \dpir/ and \diffpir/), and $\sigma_{\mathrm{final}}$ is given for images in $[-1, 1]$.
    \item \pnpflow/. We adapted the algorithm to work with DDPM-based models instead of flow-matching. We used 500 iterations, and did not perform averaging over multiple samples as we did not find it to significantly improve performance. We swept over the learning-rate parameter $\alpha$ between $[0.001, 0.9]$.
    \item \reddiff/. We fixed the number of steps to 999. We swept the overall learning rate over $[0.001, 0.5]$ and the weight of the regularization term between $[0.01, 10]$, while the data-fidelity weight was fixed to 1.
    \item \sitcom/. We fixed the number of annealing iterations to 20 and of inversion iterations to 20. We swept the learning rate of the model inversion over $[0.001, 0.1]$ and $\delta\in[0.001, 0.1]$ which controls a threshold on the loss for early stopping.
    \item \dps/. We fixed the number of iterations to 1000, varied $\zeta \in [0, 1]$ (this is the noise parameter which interpolates between DDPM and DDIM), and swept over the data-fidelity gradient's strength between $[0.1, 12]$.
\end{itemize}

In \cref{tab:ffhq-comp} (which complements \cref{tab:big-comp} in the main text) we compare all the algorithms we considered on different inverse problems for noise-levels $0.05$ on the FFHQ dataset.
Note for both tables we took the configuration among the ones used for testing (which were chosen because of good hyperparameters on the validation set) with the minimum LPIPS. 
This is an arbitrary choice in order to present tabular-form results, but looking at the pareto front gives better understanding of the trade-offs between the different methods.

\paragraph{Sensitivity to the search budget.}
All methods were tuned with 60 BO trials, the first 30 drawn from a Sobol sequence. To check whether this budget is sufficient, we also tuned \daps/, \daps/+\trueprox/, \pnpflow/, \reddiff/, \sitcom/ and \dps/ with half of it (30 trials, 15 Sobol). In 37 of the 48 (method, dataset, task) settings, the best-PSNR result changed by at most $0.05$\,dB. The largest changes are for \reddiff/, which gained $0.2$--$0.8$\,dB on the deblurring tasks with the larger budget, and \sitcom/ on FFHQ deblurring, which moved by up to $0.4$\,dB in either direction. The search therefore has a residual variance of a few tenths of a dB, and smaller differences between methods should be read with care. Results depend mostly on the search space itself: the ranges are fixed in advance, and sometimes the selected hyperparameters lie on the boundary of the range, so wider ranges might yield better configurations.

\begin{table}
    \caption{FFHQ, $\sigma_y=0.05$. (Same as \cref{tab:big-comp} but on FFHQ). We pick the HP sample from the Pareto front with the lowest LPIPS. Empty cells: \dpir/ and \diffpir/ require a closed-form proximal step and so do not apply to non-linear tasks. NFEs are doubled for methods requiring backward passes. \texttt{Fast} \trueprox/ applies \trueprox/ ($K=4$) at the last outer iteration only; \dpir/ runs 8 outer iterations on the blurs and 20 on SR and inpainting, hence two NFE values.SSIM is reported in \cref{tab:ssim}. {\color{blue}Blue}: LPIPS at least $0.005$ below the baseline.}
    \label{tab:ffhq-comp}
    \setlength{\tabcolsep}{2pt}
    \resizebox{\linewidth}{!}{%
    \begin{tabular}{lccccccccccccc}
        \toprule
        \multirow{2}{*}{Algorithm} & \multicolumn{2}{c}{Motion Blur} & \multicolumn{2}{c}{Gaussian Blur} & \multicolumn{2}{c}{SR $\times 4$} & \multicolumn{2}{c}{Inpainting} & \multicolumn{2}{c}{HDR} & \multicolumn{2}{c}{Phase Retrieval} & \multirow{2}{*}{NFE} \\
        \cmidrule(lr){2-3} \cmidrule(lr){4-5} \cmidrule(lr){6-7} \cmidrule(lr){8-9} \cmidrule(lr){10-11} \cmidrule(lr){12-13}
        & PSNR & lpips & PSNR & lpips & PSNR & lpips & PSNR & lpips & PSNR & lpips & PSNR & lpips & \\
        \midrule
        \dpir/                       & 28.5 & 0.323 & 29.1 & 0.277 & 27.4 & 0.375 & 30.2 & 0.263 & --    & --    & --    & --    & 8\,/\,20 \\
        ~~+ \trueprox/               & 28.3 & \textcolor{blue}{0.279} & 28.5 & 0.295 & 25.8 & \textcolor{blue}{0.338} & 29.1 & \textcolor{blue}{0.245} & -- & -- & -- & -- & 32\,/\,80 \\
        ~~+ \texttt{Fast} \trueprox/ & 28.7 & \textcolor{blue}{0.280} & 29.1 & \textcolor{blue}{0.271} & 28.4 & \textcolor{blue}{0.297} & 31.7 & \textcolor{blue}{0.208} & -- & -- & -- & -- & 11\,/\,23 \\
        \hdashline\vspace*{-0.25cm}\\ 
        \diffpir/                    & 27.5 & 0.252 & 27.4 & 0.247 & 27.2 & 0.272 & 28.9 & 0.236 & --    & --    & --    & --    & 20 \\
        ~~+ \trueprox/               & 27.6 & \textcolor{blue}{0.244} & 27.3 & \textcolor{blue}{0.235} & 26.8 & \textcolor{blue}{0.257} & 28.7 & \textcolor{blue}{0.211} & -- & -- & -- & -- & 80 \\
        ~~+ \texttt{Fast} \trueprox/ & 27.7 & 0.252 & 27.3 & 0.245 & 27.2 & 0.270 & 28.9 & 0.236 & -- & -- & -- & -- & 23 \\
        \hdashline\vspace*{-0.25cm}\\ 
        \daps/                       & 27.5 & 0.273 & 28.1 & 0.270 & 27.5 & 0.279 & 29.8 & 0.210 & 22.3    & 0.228    & 25.4  & 0.313 & 1200 \\
        ~~+ \trueprox/               & 27.7 & \textcolor{blue}{0.252} & 28.1 & \textcolor{blue}{0.254} & 27.6 & \textcolor{blue}{0.267} & 30.3 & \textcolor{blue}{0.186} & 22.4 & \textcolor{blue}{0.217} & 25.6 & \textcolor{blue}{0.287} & 1200 \\
        \hdashline\vspace*{-0.25cm}\\
        \pnpflow/                    & 27.0 & 0.336 & 28.2 & 0.315 & 27.8 & 0.323 & 31.8 & 0.170 & 20.4 & 0.232 & 25.3 & 0.302 & 500 \\
        \reddiff/                    & 26.6 & 0.324 & 27.5 & 0.305 & 27.7 & 0.321 & 30.4 & 0.198 & 20.7    & 0.244    & 22.3 & 0.421 & 1000 \\
        \sitcom/                     & 27.8 & 0.272 & 28.6 & 0.267 & 28.2 & 0.280 & 28.6 & 0.254 & 23.3  & 0.250 & 24.3  & 0.335 & 800 \\
        \dps/                        & 26.4 & 0.252 & 27.0 & 0.230 & 27.0 & 0.251 & 30.4 & 0.173 & 22.4    & 0.204    & 20.7    & 0.408    & 2000 \\
        \bottomrule
    \end{tabular}}
\end{table}

\section{Additional Results for PnP Experiments}\label{app:add_results}

\begin{table}
    \caption{SSIM ($\uparrow$), $\sigma_y=0.05$, at the same hyperparameter configurations as the PSNR and LPIPS results of \cref{tab:big-comp,tab:ffhq-comp,tab:reduce_compute_ffhq,tab:reduce_compute_imagenet} (the sample of the test Pareto front with the lowest LPIPS), on the same 70 test images. The distilled student replaces the denoiser at every outer iteration (\cref{app:reduce_compute}). Empty cells: \dpir/ and \diffpir/ require a closed-form proximal step and so do not apply to non-linear tasks.}
    \label{tab:ssim}
    \centering
    \setlength{\tabcolsep}{2pt}
    \resizebox{\linewidth}{!}{%
    \begin{tabular}{lccccccccccccc}
        \toprule
        \multirow{2}{*}{Algorithm} & \multicolumn{6}{c}{FFHQ} & \multicolumn{6}{c}{ImageNet} & \multirow{2}{*}{NFE} \\
        \cmidrule(lr){2-7} \cmidrule(lr){8-13}
        & Motion & Gauss. & SR $\times 4$ & Inpaint. & HDR & Phase & Motion & Gauss. & SR $\times 4$ & Inpaint. & HDR & Phase & \\
        \midrule
        \dpir/                       & 0.770 & 0.810 & 0.700 & 0.784 & -- & -- & 0.676 & 0.670 & 0.638 & 0.754 & -- & -- & 8\,/\,20 \\
        ~~+ \trueprox/               & 0.806 & 0.800 & 0.784 & 0.861 & -- & -- & 0.711 & 0.672 & 0.662 & 0.749 & -- & -- & 32\,/\,80 \\
        ~~+ \texttt{Fast} \trueprox/ & 0.814 & 0.817 & 0.812 & 0.885 & -- & -- & 0.708 & 0.675 & 0.631 & 0.801 & -- & -- & 11\,/\,23 \\
        ~~+ distilled student        & 0.811 & 0.792 & 0.762 & 0.872 & -- & -- & 0.712 & 0.672 & 0.666 & 0.791 & -- & -- & 8\,/\,20 \\
        \hdashline\vspace*{-0.25cm}\\
        \diffpir/                    & 0.775 & 0.771 & 0.774 & 0.829 & -- & -- & 0.683 & 0.666 & 0.661 & 0.730 & -- & -- & 20 \\
        ~~+ \trueprox/               & 0.780 & 0.764 & 0.758 & 0.831 & -- & -- & 0.641 & 0.583 & 0.667 & 0.738 & -- & -- & 80 \\
        ~~+ \texttt{Fast} \trueprox/ & 0.782 & 0.770 & 0.773 & 0.829 & -- & -- & 0.683 & 0.668 & 0.669 & 0.732 & -- & -- & 23 \\
        ~~+ distilled student        & 0.778 & 0.764 & 0.762 & 0.823 & -- & -- & 0.636 & 0.561 & 0.597 & 0.709 & -- & -- & 20 \\
        \hdashline\vspace*{-0.25cm}\\
        \daps/                       & 0.745 & 0.765 & 0.754 & 0.818 & 0.825 & 0.686 & 0.624 & 0.629 & 0.632 & 0.748 & 0.743 & 0.328 & 1200 \\
        ~~+ \trueprox/               & 0.774 & 0.784 & 0.774 & 0.846 & 0.841 & 0.725 & 0.677 & 0.657 & 0.662 & 0.785 & 0.765 & 0.383 & 1200 \\
        \hdashline\vspace*{-0.25cm}\\
        \pnpflow/                    & 0.772 & 0.799 & 0.787 & 0.892 & 0.871 & 0.723 & 0.622 & 0.635 & 0.638 & 0.831 & 0.774 & 0.316 & 500 \\
        \reddiff/                    & 0.750 & 0.773 & 0.787 & 0.852 & 0.824 & 0.551 & 0.623 & 0.638 & 0.644 & 0.783 & 0.706 & 0.209 & 1000 \\
        \sitcom/                     & 0.782 & 0.806 & 0.799 & 0.777 & 0.812 & 0.679 & 0.685 & 0.680 & 0.680 & 0.685 & 0.696 & 0.286 & 800 \\
        \dps/                        & 0.739 & 0.753 & 0.754 & 0.863 & 0.837 & 0.573 & 0.624 & 0.617 & 0.642 & 0.784 & 0.758 & 0.235 & 2000 \\
        \bottomrule
    \end{tabular}}
\end{table}

\begin{figure}
    \centering
    \includegraphics[width=0.99\linewidth]{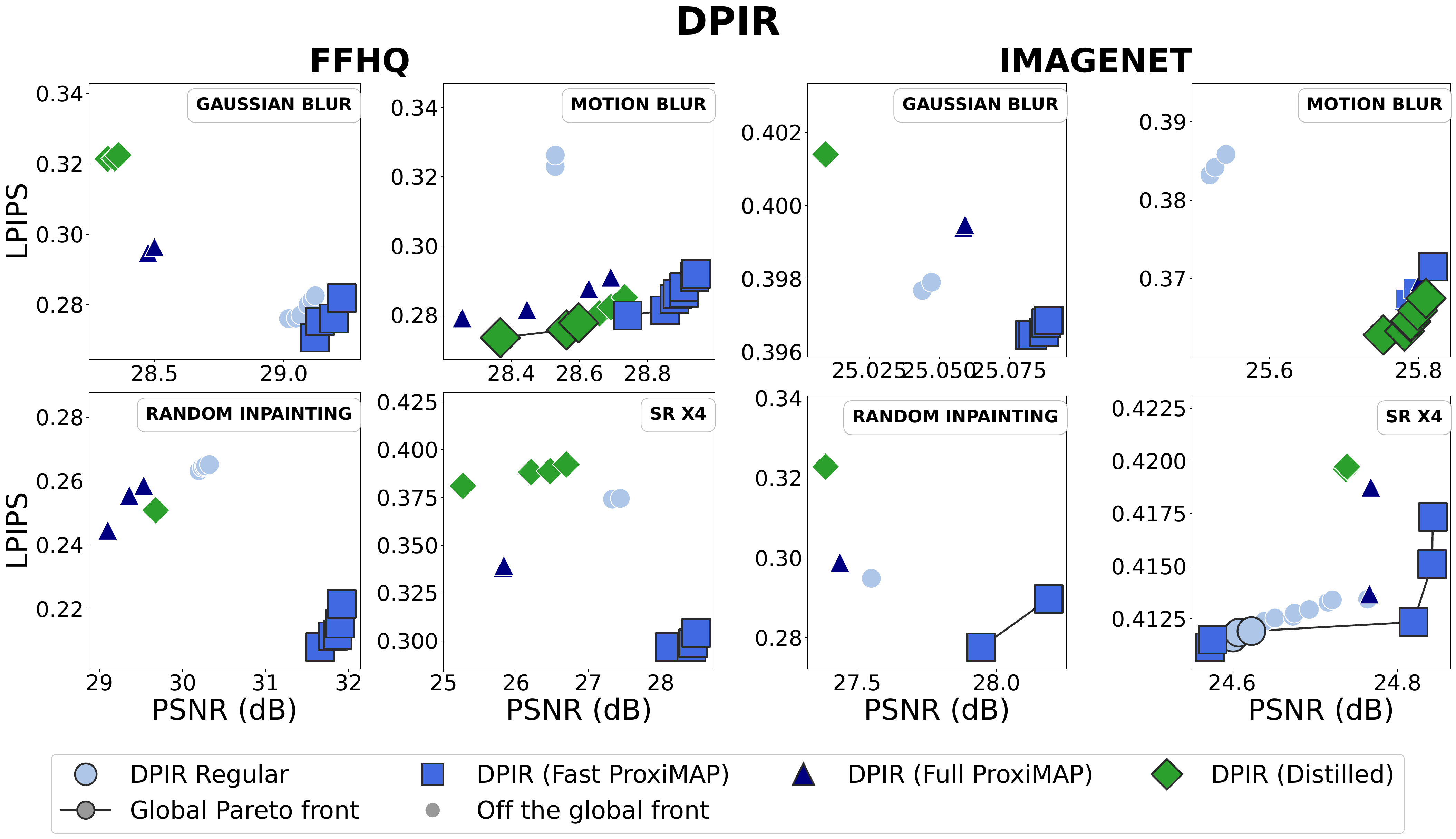}\\[0.8em]
    \includegraphics[width=0.99\linewidth]{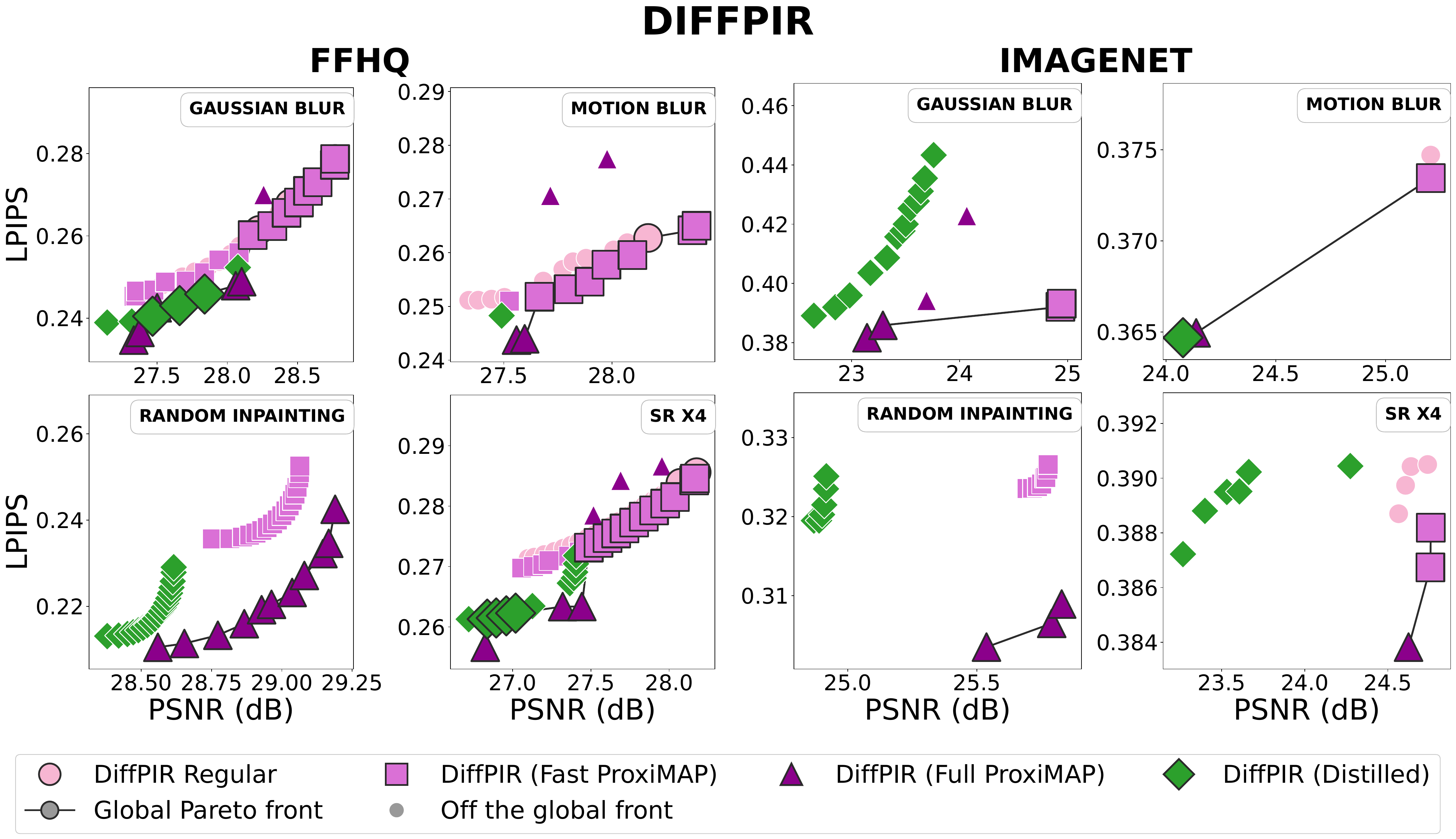}
    \caption{Pareto fronts of \dpir/ (top) and \diffpir/ (bottom) comparing all the variants (full \trueprox/ replacement, \texttt{Fast} \trueprox/ and distilled student) with the baseline. The pareto fronts (black lines and points with outline) show us the best variant for each task and algorithm.}
    \label{fig:full-replacement_comparison}
\end{figure}

First, the results of fully replacing the MMSE denoiser with \trueprox/ in PnP algorithms are detailed in \cref{tab:ffhq-comp,fig:full-replacement_comparison}.
On FFHQ, \dpir/+\trueprox/ trades PSNR for perceptual quality: LPIPS improves on three of the four tasks, but PSNR drops on all four, by up to $1.6$\,dB. This PSNR loss is not the price of sharpening itself: \texttt{Fast} \trueprox/, which calls \trueprox/ only at the last iteration, obtains similar or larger LPIPS gains on FFHQ without any PSNR loss (\cref{app:reduce_compute}). We attribute it to \dpir/'s schedule, which does not match the residual noise of its iterates. The mechanism is the one described in \cref{sec:proximap-design}: \trueprox/ takes the noise level handed to it by the outer \dpir/ loop as the noise actually present in the iterate. However \dpir/'s logarithmically spaced schedule does not take into account the noise in the image (as opposed to \diffpir/ that adds back noise at each step), so the score error compounds over the $K$ inner steps. The MMSE denoiser is robust to this mismatch; \trueprox/, by design (\cref{sec:proximap-design}), is not. On ImageNet, the search largely avoids this regime: on three of the four tasks it selects $\sigma_K$ at or near the upper end of its range ($0.08$ to $0.1$), so \trueprox/ only brings each iterate down to this level and reduces to a single MMSE call once the scheduled noise level falls below it, and the results stay close to the baseline. On motion blur, the only ImageNet task where the search selects a small $\sigma_K$, \dpir/+\trueprox/ improves both metrics. The hybrid variant introduced in \cref{sec:proximap-pnp} removes the failure mode by relying on the MMSE denoiser precisely during the early iterations, where the mismatch is largest.

For \diffpir/, which re-injects Gaussian noise at the scheduled level at every step, the noise level handed to \trueprox/ is close to the noise actually present in the iterate, and full replacement behaves as intended: LPIPS improves in seven of the eight (task, dataset) settings, by $0.008$ to $0.025$, and is tied in the last one. On FFHQ, this costs at most $0.4$\,dB of PSNR. On the ImageNet blurs, the trade-off is much steeper: PSNR drops by $1.1$ and $1.7$\,dB for an LPIPS gain of about $0.01$, and even the best-PSNR configuration of \diffpir/+\trueprox/ remains $0.8$ to $1.1$\,dB below the baseline. This is consistent with the reliance of \trueprox/ on an accurate score (\cref{sec:proximap-design}), which is harder to achieve on the broader ImageNet distribution.

For \daps/, \trueprox/ improves LPIPS in all eight linear (task, dataset) settings and is never worse in PSNR (\cref{tab:big-comp,tab:ffhq-comp}). \Cref{fig:daps_inpainting} shows two ImageNet inpainting examples, where \trueprox/ restores textures that the MMSE-based baseline smooths out.

\begin{figure}
    \centering
    \includegraphics[width=\linewidth]{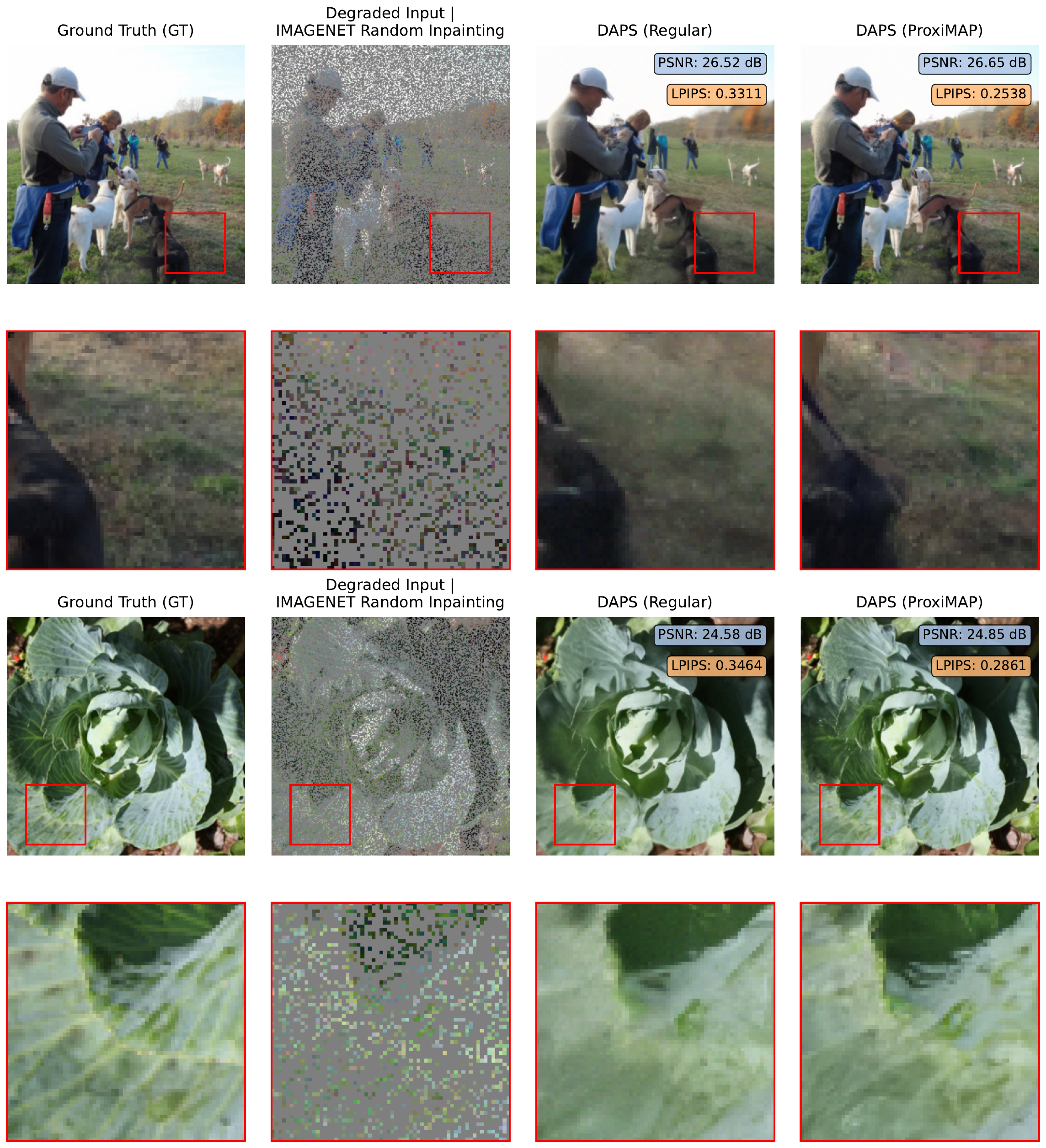}
    \caption{Qualitative results of \daps/ and its full-replacement variant on random inpainting. The \trueprox/ variant has a better texture on the image compared to the smoothed MMSE result.}
    \label{fig:daps_inpainting}
\end{figure}

\begin{figure}
    \centering
    \includegraphics[width=\linewidth]{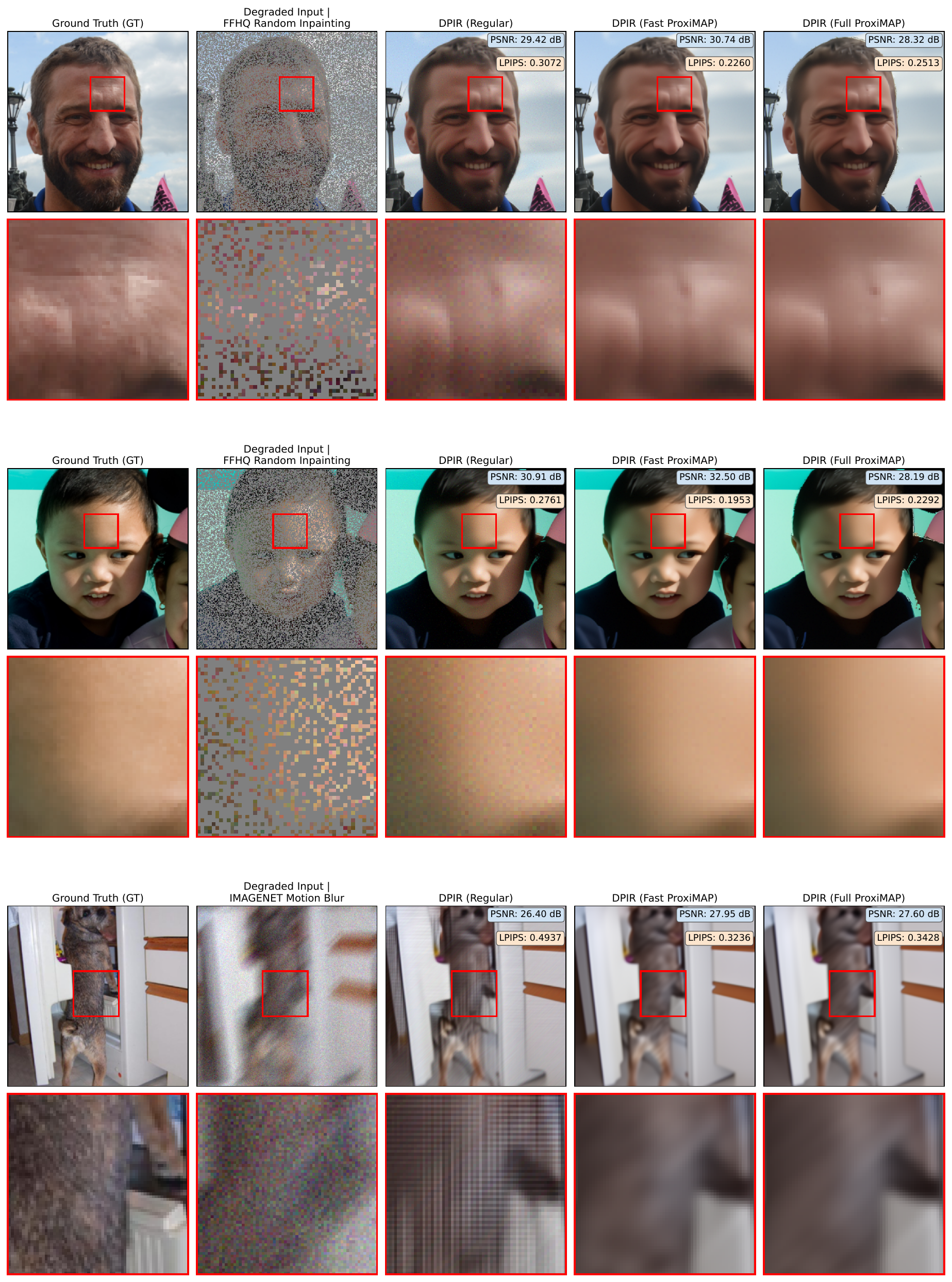}
    \caption{\dpir/ on FFHQ random inpainting (top two rows) and ImageNet motion blur (bottom row). The MMSE baseline leaves residual noise and stripe artifacts, which \texttt{Fast} \trueprox/ removes. Full replacement removes them as well, but \texttt{Fast} \trueprox/ has the best PSNR and LPIPS on all three images.}
    \label{fig:dpir_gaussian}
\end{figure}

\begin{figure}[t]
    \centering
    \includegraphics[width=\linewidth]{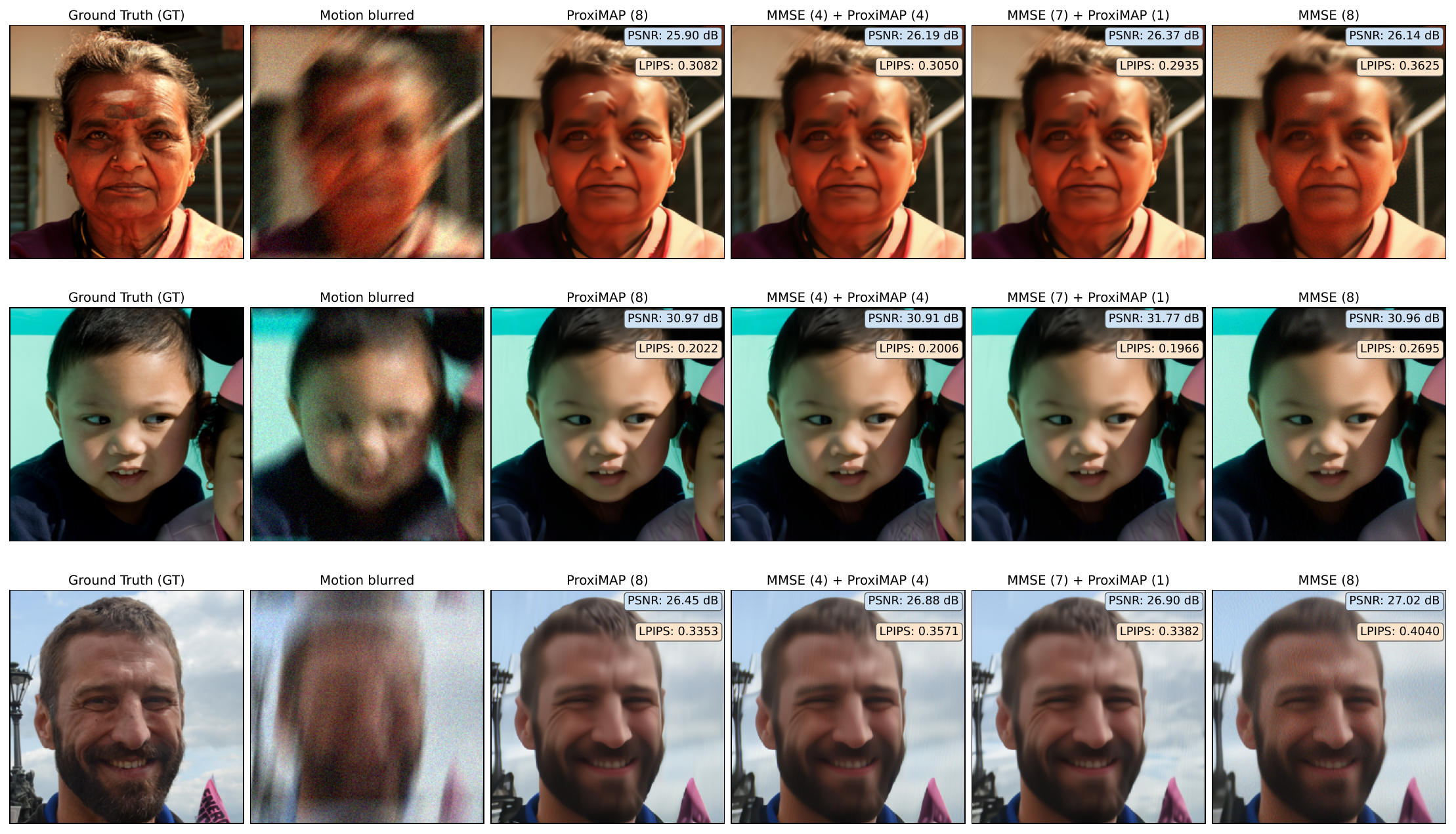}
    \caption{Qualitative visualization for different switch times on motion blur. Panel titles indicate the number of outer iterations performed at each stage with the corresponding algorithm (\trueprox/ or MMSE) for \dpir/ on FFHQ. Images are sharper for the \trueprox/ variants than the regular MMSE one.}
    \label{fig:motion_blur_switch}
\end{figure}

\begin{figure}
    \centering
    \includegraphics[width=\linewidth]{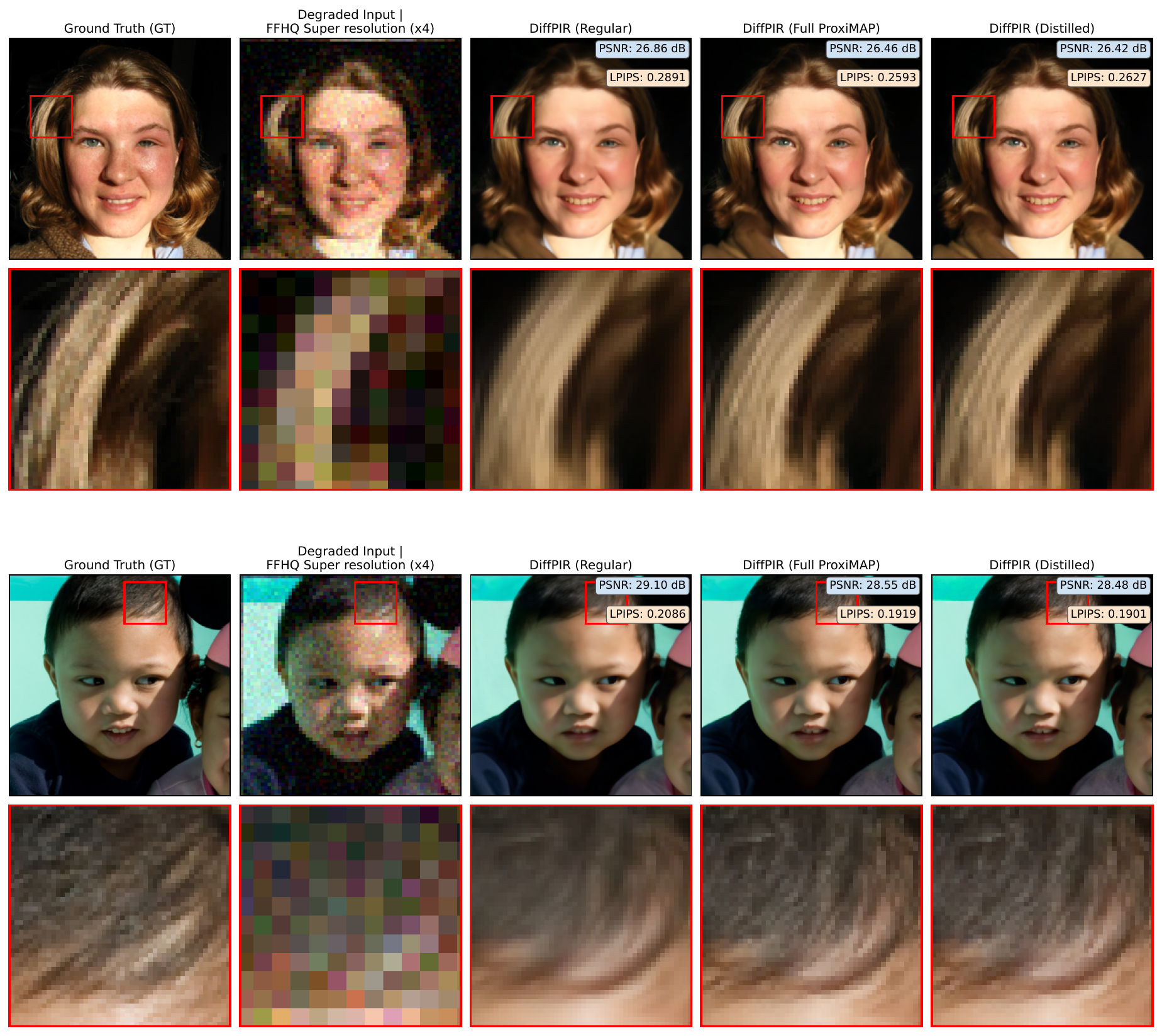}
    \caption{Super resolution $\times 4$ using \diffpir/ on FFHQ. Here, the distilled model is very close to the full \trueprox/ variant. The second row is a zoom in the red square.}
    \label{fig:diffpir_sr}
\end{figure}

\paragraph{Failure cases and how \texttt{Fast} \trueprox/ resolves them}\label{app:failure}

The hyperparameter search did not always converge to good configurations within the allotted budget for \dpir/+\trueprox/. This issue occurred for the super-resolution task on FFHQ, where \cref{fig:dpir-sr-fail} shows high-frequency artifacts along image edges in a few images produced by \dpir/+\trueprox/. These artifacts have a characteristic appearance reminiscent of the early stages of the cartoon effect produced by the iteration of \cref{eq:original-algo} (\cref{sec:diagnosis}). Fast \trueprox/ resolves this issue cleanly, as visible in the same figure: by relying on the standard MMSE denoiser for the early iterations, the hybrid variant avoids entering the regime where score errors generate artifacts, while still benefiting from \trueprox/'s sharpening in the late iterations. This is consistent with the noise-matching design analysis of \cref{sec:proximap-design}.

\paragraph{HDR.}
HDR is non-linear, so \dpir/ and \diffpir/ do not apply and the comparison rests on \daps/ alone: one entry per dataset. \trueprox/ improves LPIPS on both datasets ($0.310$ vs.\ $0.317$ on ImageNet, $0.217$ vs.\ $0.228$ on FFHQ), while PSNR stays within $0.11$\,dB of the baseline ($19.5$ vs.\ $19.6$\,dB on ImageNet, $22.4$ vs.\ $22.3$\,dB on FFHQ). All methods score low on this task ($18.0$--$19.8$\,dB on ImageNet): the forward model saturates, and clipped pixels leave the data-fidelity term flat, so the reconstruction there is driven by the prior alone. On phase retrieval, the other non-linear task, \trueprox/ improves both metrics on both datasets ($15.9$ to $16.3$\,dB and $0.560$ to $0.536$ LPIPS on ImageNet, $25.4$ to $25.6$\,dB and $0.313$ to $0.287$ on FFHQ), and \daps/+\trueprox/ is the best of all methods in both metrics. ImageNet phase retrieval nevertheless remains hard for every method ($12.6$--$16.3$\,dB).

\section{Reducing computing costs}\label{app:reduce_compute}

As stated above, the main drawback of \trueprox/ is its computational cost since it costs around $4$ times more NFE than running a simple MMSE denoiser. Although NFE are much lower than other methods such as PnPFlow \citep{pnpflow}, ReDDiff \citep{reddiff}, SITCOM \citep{sitcom} and DPS \citep{dps}, we believe that reducing this computational overhead is important to make \trueprox/ more practical. We therefore investigate a simple one-step distillation of \trueprox/ and a hybrid scheduler in PnP algorithms. We focus our experiments on \diffpir/ and \dpir/ to study the impact of both strategies on reconstruction quality and computational cost.

\paragraph{Distillation}
To keep computation low and preserve \trueprox/'s generalisation, we train a single-step student network to reproduce the outputs of \trueprox/. However inside PnP algorithms, the denoiser does not receive only gaussian-noisy images, but intermediate iterates that have non-gaussian noise and artefacts. This is especially the case for \dpir/ since the outputs of the data-fidelity terms are not gaussian and given as input to the denoiser. Therefore we train the student on a mix of gaussian-noisy images and intermediate iterates from \dpir/ and \diffpir/ (the images passed to the denoiser). This allows us to use the same student model in all tasks and both PnP algorithms.

We train two models, one for FFHQ and one for ImageNet, based on the MMSE denoisers for each of these datasets: the model of \citet{choi21ilvr} for FFHQ and the (larger) model of \citet{dhariwal2021diffusion} for ImageNet, using their implementation in the \texttt{DeepInverse} library~\citep{tachella2025deepinverse}. Each model's training set is split in half between gaussian-noisy images and intermediate iterates of \dpir/+\trueprox/ and \diffpir/+\trueprox/ on the four tasks considered, since these images have residual non-gaussian noise and artefacts that the model has to learn to denoise. The student $D_\theta$ is initialized from the corresponding MMSE denoiser and finetuned to match the output of \trueprox/ (frozen with K=4, $\tau=\sigma^2$ and $\sigma_{end}=0.001$). For a training pair $(y, \sigma)$, either a Gaussian-noisy image $y = x + \sigma\varepsilon$ with $\varepsilon \sim \cN(0, I)$ or a PnP iterate with its scheduled noise level $\sigma$, we minimize
\begin{equation}
    \mathcal{L}(\theta) = \mathbb{E}_{(y, \sigma)} \left[ \left\| D_\theta(y, \sigma) - \mathrm{P}(y, \sigma) \right\|_1 \right],
\end{equation}
where $\mathrm{P}(y, \sigma)$ denotes the output of \trueprox/, which is kept frozen during training. Once trained, the student can be used as a drop-in replacement for \trueprox/, reducing its cost to that of a standard denoiser. As \trueprox/ is expensive to run, we do not compute its output on the fly during training. Instead, we precompute a fixed dataset of pairs $(y, P(y,\sigma))$ by applying \trueprox/ (with $K=4$ steps) to noisy versions of the training images, using 6 independent noise realizations for FFHQ and 4 for ImageNet. Except for the number of trained images ($60\,000$ for FFHQ and $100\,000$ for ImageNet), and batch sizes ($128$ for FFHQ and $32$ for ImageNet), all the settings are shared across datasets and detailed in \cref{tab:distill_hparams}.

\begin{table}[ht]
    \centering
    \caption{Training details for the distilled students. Noise levels are for images in $[0, 1]$; batch sizes are global, over 4 GPUs. EMA follows the power-function scheme of \citet{edm2}.}
    \label{tab:distill_hparams}
    \vspace*{0.1cm}
    \small
    \begin{tabular}{@{}lll@{}}
        \toprule
        & FFHQ & ImageNet \\
        \midrule
        Gaussian pairs & $60$k images $\times$ $6$ seeds & $100$k images $\times$ $4$ seeds \\
        Batch size & $128$ & $32$ \\
        Augmentation & none & horizontal flip \\
        \midrule
        \multicolumn{3}{@{}l}{\textit{Shared by both datasets}} \\
        Noise level $\sigma$ & \multicolumn{2}{l}{log-uniform in $[10^{-3}, 78.7]$} \\
        Trajectory pairs & \multicolumn{2}{l}{\dpir/ and \diffpir/ + \trueprox/ iterates, $4$ tasks ($8$ runs, equal weight),} \\
        & \multicolumn{2}{l}{$4$k images per run, all steps ($544$k pairs)} \\
        Batch composition & \multicolumn{2}{l}{half Gaussian pairs, half trajectory pairs} \\
        Teacher & \multicolumn{2}{l}{\trueprox/, $K=4$, $\tau=\sigma^2$, $\sigma_K=10^{-3}$} \\
        Initialization & \multicolumn{2}{l}{pretrained MMSE denoiser (DiffUNet) used by the teacher} \\
        Loss & \multicolumn{2}{l}{$\ell_1$ to the teacher output} \\
        Optimizer & \multicolumn{2}{l}{Adam, learning rate $3\times10^{-5}$, cosine decay to $0$} \\
        EMA $\sigma_{\text{rel}}$ & \multicolumn{2}{l}{$0.10$} \\
        Images seen & \multicolumn{2}{l}{$10$M} \\
        \bottomrule
    \end{tabular}
\end{table}

First, we verify that our distilled model has good performance on denoising compared to the teacher \trueprox/ and the initial MMSE denoiser. \Cref{fig:distill_denoise} shows an example of denoising on a single image at the top and a comparison of PSNR and LPIPS between the MMSE denoiser, the \trueprox/ teacher and the distilled model over multiple noise levels. For each considered $\sigma$, we compute the PSNR and LPIPS with respect to the original image and average those values over our 100 held-out images (the same images used in the hyperparameter search and validation for PnP methods). The results are very good both visually and quantitatively: the student's output image is very similar to the teacher's and it is much sharper than the MMSE's output, and the PSNR and LPIPS curves of the student and teacher nearly overlap for every noise level tested. We then evaluate the student inside \dpir/ and \diffpir/, together with the hybrid schedule introduced below, in the Results paragraph.

\begin{figure}
    \centering
    \includegraphics[width=\linewidth]{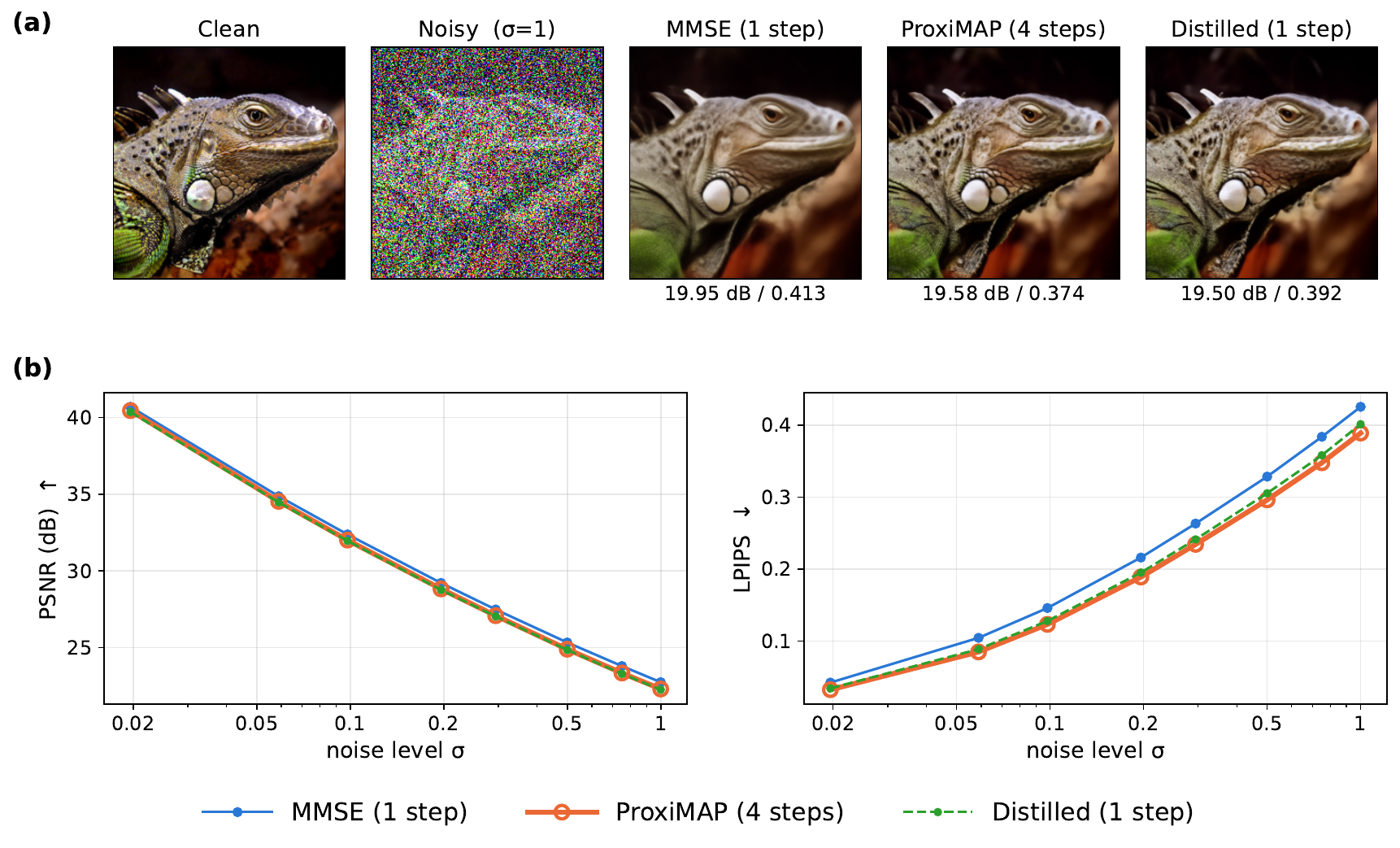}
    \caption{Gaussian denoising on ImageNet. (a) One image at $\sigma=1$: MMSE (1 NFE), the \trueprox/ teacher ($K=4$, 4 NFE) and the distilled student (1 NFE), with PSNR / LPIPS under each output. (b) Mean PSNR and LPIPS over 100 images against $\sigma$.}
    \label{fig:distill_denoise}
\end{figure}

\begin{figure}
    \centering
    \includegraphics[width=\linewidth]{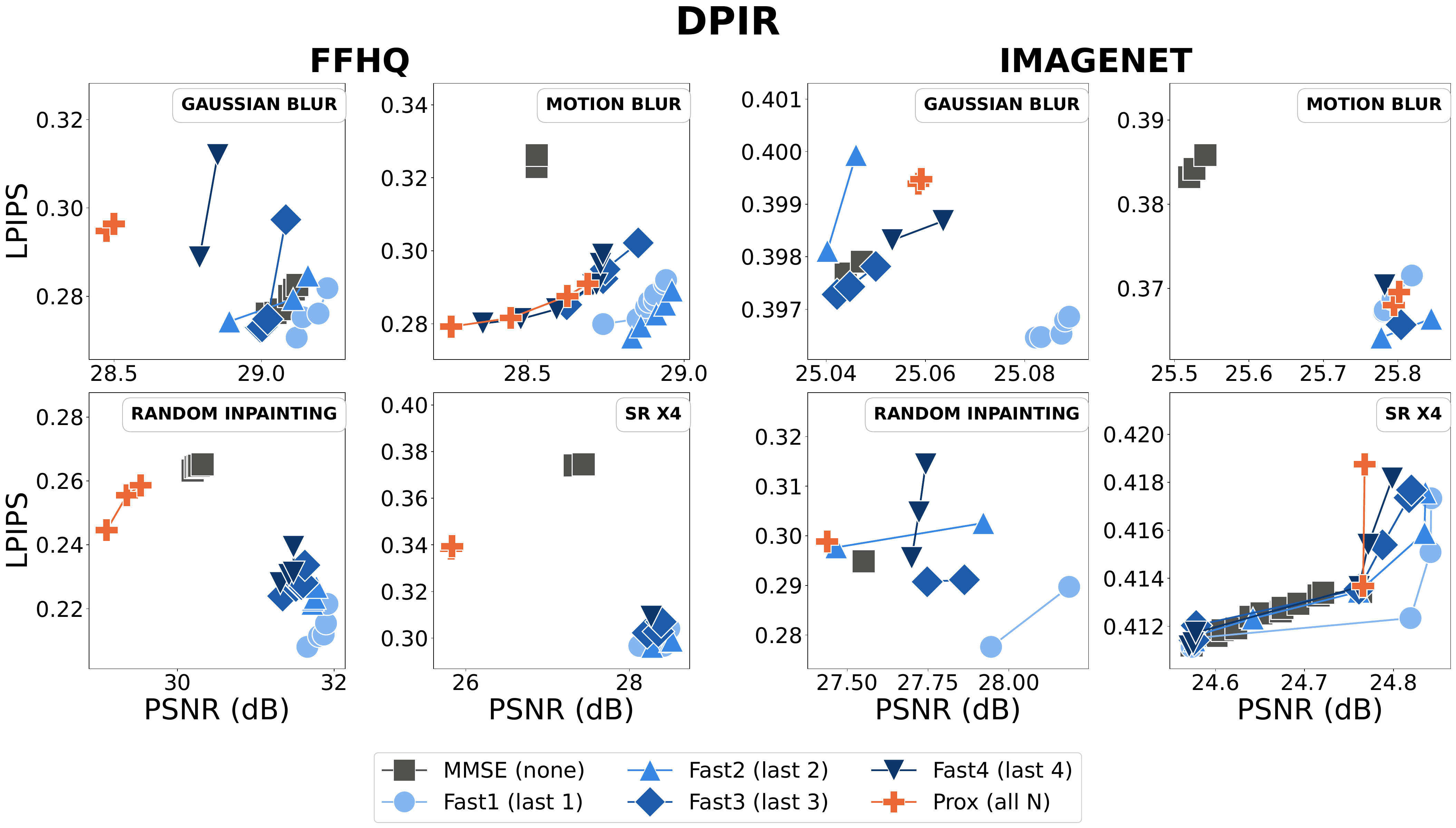}
    \caption{Comparison of many hybrid variants of \dpir/. Switching at the last or second to last step gives the best results.}
    \label{fig:switch_steps_dpir}
\end{figure}

\begin{figure}
    \centering
    \includegraphics[width=\linewidth]{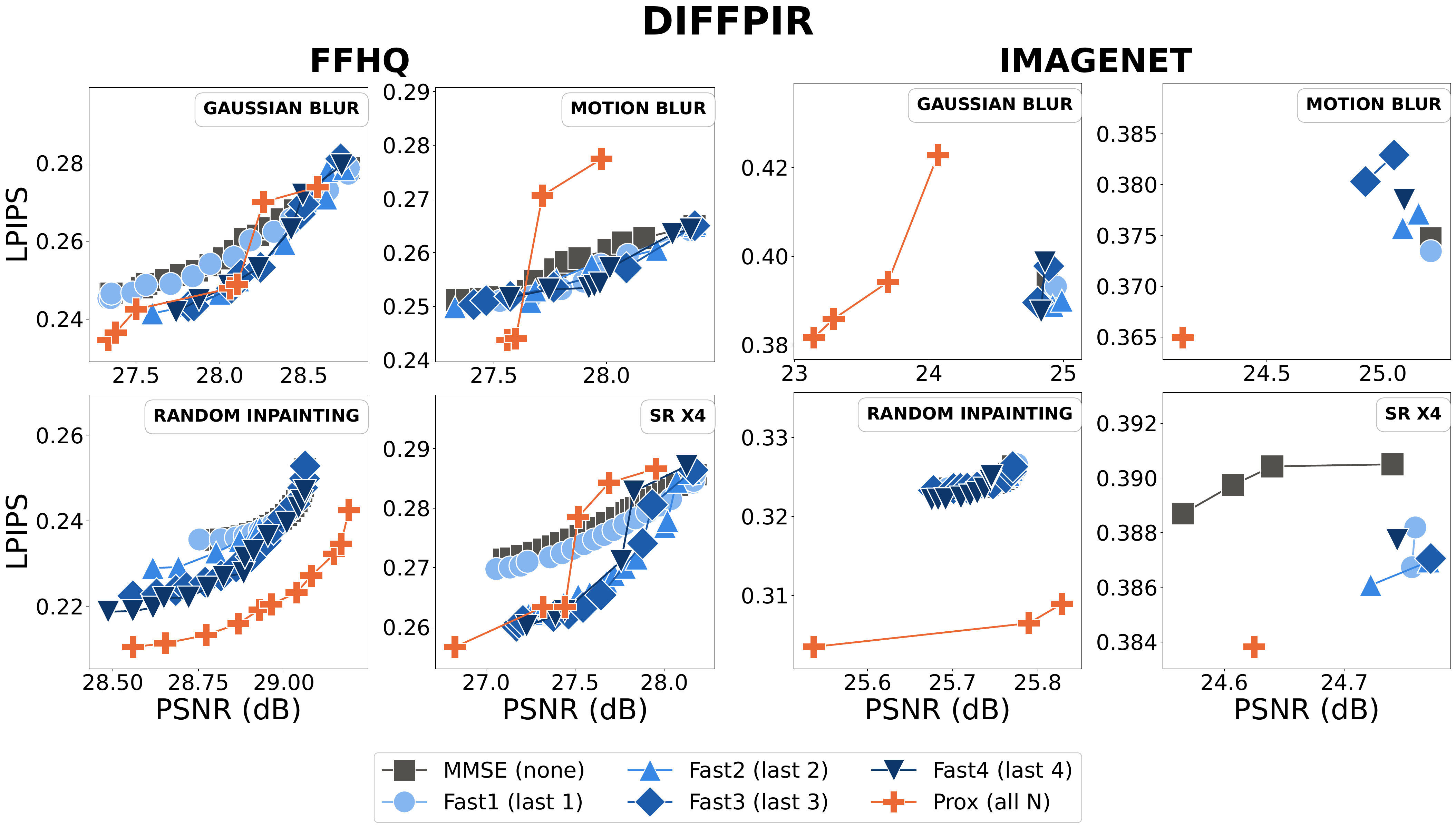}
    \caption{Same as \cref{fig:switch_steps_dpir}, but on \diffpir/.}
    \label{fig:switch_steps_diffpir}
\end{figure}

\paragraph{Hybrid schedule.} Distillation is simple and effective but it is not always possible to train a model. Moreover, we noticed that distillation brings only limited gains over the baseline for \dpir/, especially on FFHQ (even when training it partly with those intermediate images). Therefore we set out to find a training-free method to reduce the cost while showing good performance in this setting. As described in \cref{sec:algorithm}, we decided to use a hybrid schedule in the PnP outer loop: We start by using the MMSE denoiser for the first iterations, and switch to \trueprox/ for the last $n$ iterations.

This choice is motivated by an empirical observation on \dpir/. The image passed to the denoiser is the output of the data-fidelity proximal step. This image is still noisy, but the residual noise is not gaussian. The MMSE denoiser, trained on Gaussian denoising, struggles to remove it, and its outputs retain residual noise and stripe artifacts (\cref{fig:dpir_gaussian}). \trueprox/, which injects gaussian noise in its inner loop, seems more capable of denoising these proximal outputs, and using it for a single step at the end of the process (the last step of PnP algorithms) is enough to recover a clean image without artifacts. This is also consistent with \cref{sec:algorithm}'s schedule mismatch. As MMSE is more robust than \trueprox/ when there is a mismatch between the scheduled noise and the noise actually present in the image, we use it for the early iterations. We observe in \cref{fig:switch_steps_dpir} that we have the best performances overall when switching only at the last step, where the mismatch is small and \trueprox/ can be used reliably; \cref{fig:motion_blur_switch} shows the corresponding images on motion blur. \diffpir/ is less affected by the hybrid schedule (\cref{fig:switch_steps_diffpir}) because it reinjects gaussian noise at the level prescribed by the scheduler and its last noise level is particularly small ($\sigma_{end}=0.005$) which limits the impact of the final \trueprox/ step.

With $K=4$ inner iterations and $8$ or $20$ outer steps, we get NFEs of $11$ or $23$ instead when we replace only the last iteration. We refer to this variant as \texttt{Fast} \trueprox/.

\paragraph{Results}
We now compare both strategies when plugged into \diffpir/ and \dpir/. \Cref{tab:reduce_compute_ffhq,tab:reduce_compute_imagenet} report their performance alongside the MMSE baselines and the full \trueprox/ variant.

For \diffpir/, distillation brings the biggest perceptual gains. On FFHQ, the student improves LPIPS over the baseline on all tasks but motion blur, where it is tied, and is within $0.005$ LPIPS of the teacher on all four tasks. Its PSNR is on par with the baseline on the two blurs and $0.2$ to $0.4$\,dB lower on super-resolution and inpainting, so perception improves at a small cost in fidelity. \Cref{fig:diffpir_sr} shows two super-resolution examples, where the student output is very close to that of the teacher. On ImageNet, its gains are less consistent: LPIPS improves only on the two blurs, and PSNR drops by $0.9$ to $2.2$\,dB on all four tasks. \texttt{Fast} \trueprox/, on the other hand, has almost no effect on \diffpir/, which is expected given the small noise level of its last step: its results are tied with the baseline in 14 of the 16 cells, the two others being gains of $0.2$\,dB, within the variance of the hyperparameter search (\cref{app:pnp}).

It is the opposite for \dpir/. On FFHQ, \texttt{Fast} \trueprox/ is the best variant: It improves on the baseline for all but one cell (Gaussian blur PSNR, tied) with gains of up to $1.5$\,dB and $0.078$ LPIPS. It also outperforms the full \trueprox/ variant on all but one metric (motion blur LPIPS, tied), while cutting the computational cost by half (11 NFE vs 23). In contrast, the distilled student lowers PSNR in all tasks, and improves LPIPS on motion blur and inpainting but degrades it on the two other tasks. Like the full \trueprox/ variant, which also degrades PSNR on all tasks, the student is exposed to the mismatch between scheduled and actual noise of the early \dpir/ iterations, while \texttt{Fast} \trueprox/ avoids it by construction. Including some of those intermediates iterates during training helped, but it was not enough to fully mitigate the issue. On ImageNet, differences are smaller, but the same trend is observed: \texttt{Fast} \trueprox/ never degrades the baseline, improving it on motion blur and inpainting, and tying it on gaussian blur and super-resolution. The student improves both metrics on motion blur, but is worse in LPIPS on super-resolution and inpainting, and in PSNR on inpainting.

Overall, neither strategy dominates in all settings. \texttt{Fast} \trueprox/ requires no training and is the clear choice for \dpir/: it is never worse than the baseline on either dataset and, on FFHQ, also outperforms the full \trueprox/ variant, for only 3 extra NFE. For \diffpir/ it has no effect, and the distilled student is the only low-cost option that improves perceptual quality: on FFHQ it nearly matches the full \trueprox/ variant in LPIPS at the cost of the MMSE baseline, but on ImageNet it loses $0.9$ to $2.2$\,dB of PSNR. When computational cost is not a constraint, the full \trueprox/ variant gives the best LPIPS for \diffpir/ on both datasets, at the price of lower PSNR (up to $1.7$\,dB on the ImageNet blurs).

\begin{table}
    \caption{FFHQ, $\sigma_y=0.05$. Comparison of the cost-reduction strategies. Same setting as \cref{tab:ffhq-comp}. For \diffpir/, $\zeta$ is tuned for each variant. \dpir/ runs 8 outer iterations on the blurs and 20 on SR and inpainting, hence two NFE values. \texttt{Fast} \trueprox/ applies \trueprox/ ($K=4$) at the last outer iteration only. SSIM is reported in \cref{tab:ssim}.}
    \label{tab:reduce_compute_ffhq}
    \centering
    \setlength{\tabcolsep}{2pt}
    \resizebox{0.744\linewidth}{!}{%
    \begin{tabular}{lccccccccc}
        \toprule
        \multirow{2}{*}{Algorithm} & \multicolumn{2}{c}{Motion Blur} & \multicolumn{2}{c}{Gaussian Blur} & \multicolumn{2}{c}{SR $\times 4$} & \multicolumn{2}{c}{Inpainting} & \multirow{2}{*}{NFE} \\
        \cmidrule(lr){2-3} \cmidrule(lr){4-5} \cmidrule(lr){6-7} \cmidrule(lr){8-9}
        & PSNR & lpips & PSNR & lpips & PSNR & lpips & PSNR & lpips & \\
        \midrule
        \dpir/                       & 28.5 & 0.323 & 29.1 & 0.277 & 27.4 & 0.375 & 30.2 & 0.263 & 8\,/\,20 \\
        ~~+ distilled student        & 28.4 & \textcolor{blue}{0.273} & 28.3 & 0.322 & 25.3 & 0.380 & 29.7 & \textcolor{blue}{0.251} & 8\,/\,20 \\
        ~~+ \texttt{Fast} \trueprox/ & 28.7 & \textcolor{blue}{0.280} & 29.1 & \textcolor{blue}{0.271} & 28.4 & \textcolor{blue}{0.297} & 31.7 & \textcolor{blue}{0.208} & 11\,/\,23 \\
        ~~+ \trueprox/               & 28.3 & \textcolor{blue}{0.279} & 28.5 & 0.295 & 25.8 & \textcolor{blue}{0.338} & 29.1 & \textcolor{blue}{0.245} & 32\,/\,80 \\
        \hdashline\vspace*{-0.25cm}\\ 
        \diffpir/                    & 27.5 & 0.252 & 27.4 & 0.247 & 27.2 & 0.272 & 28.9 & 0.236 & 20 \\
        ~~+ distilled student        & 27.5 & 0.248 & 27.3 & \textcolor{blue}{0.239} & 27.0 & \textcolor{blue}{0.262} & 28.5 & \textcolor{blue}{0.214} & 20 \\
        ~~+ \texttt{Fast} \trueprox/ & 27.7 & 0.252 & 27.3 & 0.245 & 27.2 & 0.270 & 28.9 & 0.236 & 23 \\
        ~~+ \trueprox/               & 27.6 & \textcolor{blue}{0.244} & 27.3 & \textcolor{blue}{0.235} & 26.8 & \textcolor{blue}{0.257} & 28.7 & \textcolor{blue}{0.211} & 80 \\
        \bottomrule
    \end{tabular}}
\end{table}

\begin{table}
    \caption{ImageNet, $\sigma_y=0.05$. Comparison of the cost-reduction strategies. Same setting as \cref{tab:big-comp}. For \diffpir/, $\zeta$ is tuned for each variant. \dpir/ runs 8 outer iterations on the blurs and 20 on SR and inpainting, hence two NFE values. \texttt{Fast} \trueprox/ applies \trueprox/ ($K=4$) at the last outer iteration only. SSIM is reported in \cref{tab:ssim}.}
    \label{tab:reduce_compute_imagenet}
    \centering
    \setlength{\tabcolsep}{2pt}
    \resizebox{0.744\linewidth}{!}{%
    \begin{tabular}{lccccccccc}
        \toprule
        \multirow{2}{*}{Algorithm} & \multicolumn{2}{c}{Motion Blur} & \multicolumn{2}{c}{Gaussian Blur} & \multicolumn{2}{c}{SR $\times 4$} & \multicolumn{2}{c}{Inpainting} & \multirow{2}{*}{NFE} \\
        \cmidrule(lr){2-3} \cmidrule(lr){4-5} \cmidrule(lr){6-7} \cmidrule(lr){8-9}
        & PSNR & lpips & PSNR & lpips & PSNR & lpips & PSNR & lpips & \\
        \midrule
        \dpir/                       & 25.5 & 0.383 & 25.0 & 0.398 & 24.6 & 0.412 & 27.6 & 0.295 & 8\,/\,20 \\
        ~~+ distilled student        & 25.8 & \textcolor{blue}{0.363} & 25.0 & 0.401 & 24.7 & 0.420 & 27.4 & 0.323 & 8\,/\,20 \\
        ~~+ \texttt{Fast} \trueprox/ & 25.8 & \textcolor{blue}{0.367} & 25.1 & 0.396 & 24.6 & 0.411 & 27.9 & \textcolor{blue}{0.278} & 11\,/\,23 \\
        ~~+ \trueprox/               & 25.8 & \textcolor{blue}{0.368} & 25.1 & 0.399 & 24.8 & 0.414 & 27.4 & 0.299 & 32\,/\,80 \\
        \hdashline\vspace*{-0.25cm}\\ 
        \diffpir/                    & 25.2 & 0.375 & 24.9 & 0.394 & 24.6 & 0.389 & 25.8 & 0.324 & 20 \\
        ~~+ distilled student        & 24.1 & \textcolor{blue}{0.365} & 22.7 & \textcolor{blue}{0.389} & 23.3 & 0.387 & 24.9 & 0.320 & 20 \\
        ~~+ \texttt{Fast} \trueprox/ & 25.2 & 0.373 & 24.9 & 0.392 & 24.8 & 0.387 & 25.8 & 0.324 & 23 \\
        ~~+ \trueprox/               & 24.1 & \textcolor{blue}{0.365} & 23.1 & \textcolor{blue}{0.382} & 24.6 & 0.384 & 25.5 & \textcolor{blue}{0.303} & 80 \\
        \bottomrule
    \end{tabular}}
\end{table}

\begin{figure}
    \centering
    \includegraphics[width=\linewidth]{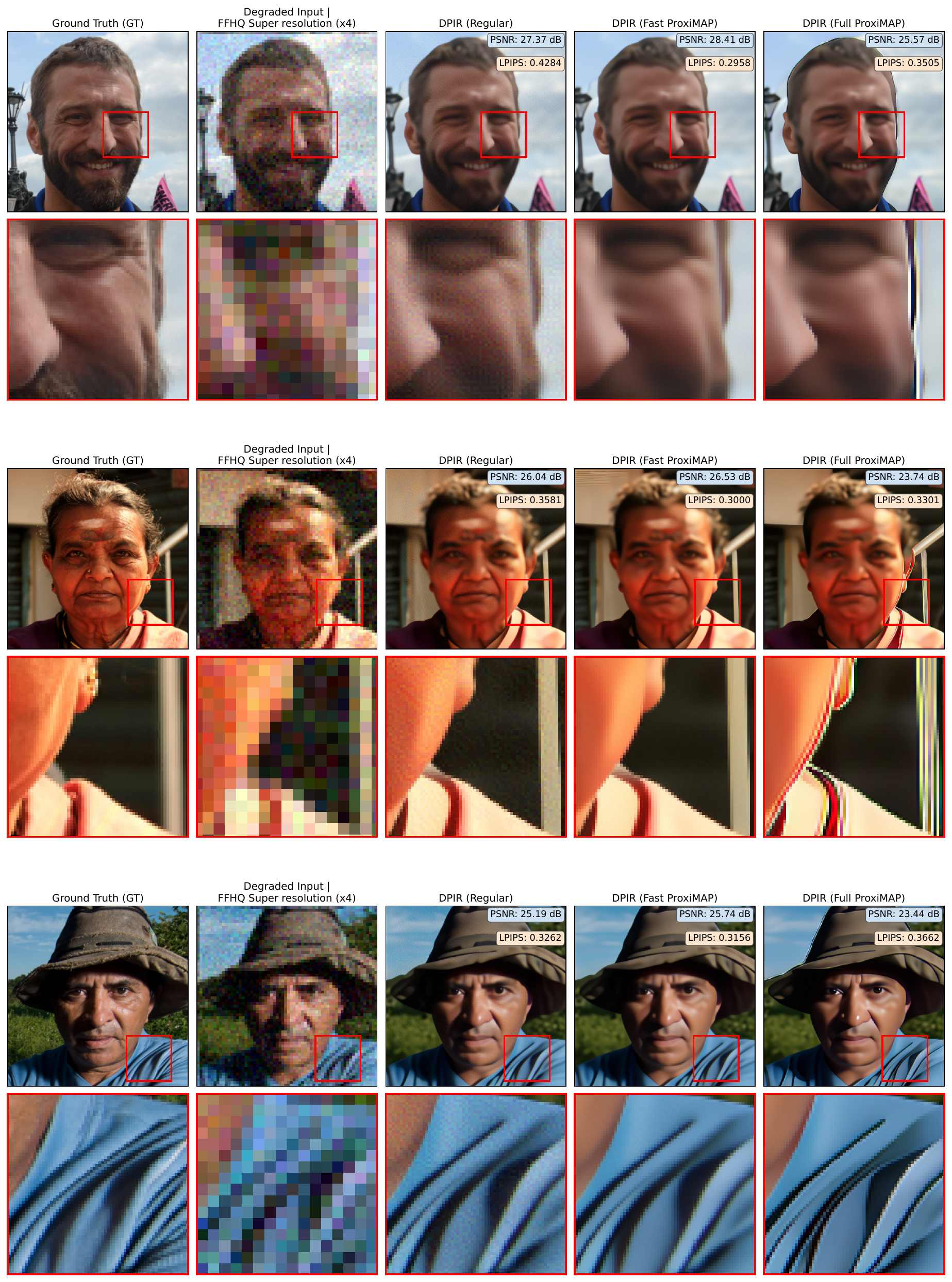}
    \caption{Failure cases for \trueprox/ in \dpir/ on the super-resolution task on FFHQ. High-frequency artifacts appear along the edges for full \trueprox/ replacement (right column) and textures have artifact on the baseline (middle column), but are removed when we use the \texttt{Fast} \trueprox/ variant (second from the right).}
    \label{fig:dpir-sr-fail}
\end{figure}

%% file: sec/theoryvpractice.tex
\cref{sec:diagnosis} established two facts that deserve unpacking. First, MAP-targeting with the schedule of~\cite{pesme2025map} fails on diffusion models trained on natural images, even though the same schedule succeeds with the exact score of a Gaussian mixture model. Second, this failure is driven by the implicit biases of diffusion model training displacing the modes of the learned distribution away from those of the data distribution. We discuss here what this means for MAP estimation under learned priors more broadly, and the precise sense in which \trueprox/ acts as principled early stopping.

\paragraph{The status of MAP under learned priors.}
The diagnosis is a statement about the \emph{learned} posterior, not about MAP estimation itself. The conditional log-likelihood values reported in \cref{fig:cartoons} are unambiguous: cartoon outputs are the legitimate maximizers of $\log p_\theta(x \mid y)$ where $p_\theta$ is the distribution implied by the learned score. They are also visually unrealistic. The reconciliation is that $p_\theta$ is not the data distribution. The implicit biases that allow diffusion models to generalize beyond the training set~\citep{kambganguli,scarvelis25closedform,yoon2023diffusion} also smooth and reshape the high-density regions in ways that do not preserve typicality. As a corollary, the MAP of the data distribution---which we cannot access---and the MAP of the learned distribution---which we can in principle compute---are different objects, and targeting the latter is not what one wants. Our results suggest a tension between generalization and MAP-faithfulness in diffusion-based priors: \cref{app:memorizing} confirms this from the other direction, showing that a diffusion model trained in the memorizing regime~\cite{biroli2024dynamical}---where the implicit smoothing is suppressed---produces artifact-free MAP estimates from~\eqref{eq:original-algo}.

\paragraph{Connections and outlook.}
The cartoon effect we exploit here was studied in its own right by~\cite{karczewski25diffusion} as a property of high-density regions of pretrained diffusion models. Our contribution is methodological: identifying it as the operative obstacle to PnP-style MAP estimation, and showing that the noise-matching constraint avoids it.

A clear caveat is that \trueprox/ does not converge to a well-defined fixed point. Its sharpness and faithfulness come from stopping in the regime where the learned score remains reliable, which is a property of the schedule rather than of any limit of the iteration. Whether a notion of MAP can be formulated for learned image priors that is both well-defined and aligned with perceptual quality remains open---perhaps via priors trained explicitly to align their high-density regions with the data manifold, or via posterior objectives less sensitive to the implicit smoothing of standard diffusion training. In the meantime, our results indicate that the most useful design lever is the iterate's relationship to the regime where the learned score is reliable: by managing this relationship explicitly, schedule-based methods can extract sharp, faithful reconstructions from existing diffusion models without requiring any change to how those models are trained.